\documentclass[letterpaper]{article}
\usepackage{aaai23}  
\usepackage{times}  
\usepackage{helvet}  
\usepackage{courier}  
\usepackage[hyphens]{url}  
\usepackage{graphicx} 
\usepackage{natbib}  
\usepackage{caption} 
\usepackage{algorithm}
\usepackage{algorithmic}
\usepackage{breqn}
\usepackage{float}

\nocopyright

\usepackage{newfloat}
\usepackage{listings}
\DeclareCaptionStyle{ruled}{labelfont=normalfont,labelsep=colon,strut=off} 
\floatstyle{ruled}
\newfloat{listing}{tb}{lst}{}
\floatname{listing}{Listing}
\usepackage{graphicx}
\usepackage{mathtools}
\usepackage{xcolor}
\usepackage{amsmath}
\usepackage{amsthm}
\usepackage{amssymb}
\usepackage{adjustbox}
\usepackage{pifont}
\newcommand{\cmark}{\ding{51}}%
\newcommand{\xmark}{\ding{55}}%
\usepackage{url}
\usepackage{makecell}
\usepackage{multirow}\usepackage{multicol}

\newtheorem{definition}{Definition}
\newtheorem{theorem}{Theorem}

\lstdefinelanguage{pddl}
{
    keywords={
        define,
    },
    keywordstyle=[2]{\color{red}},
    morekeywords={[2]
        domain,
        problem,
        requirements,
        predicates,
        types,
        action,
        init,
        goal,
        typing,
        costs,
        functions,
    },
    keywordstyle=[3]{\color{blue}},
    morekeywords={[3]
        parameters,
        vars,
        precondition,
        effect,
        duration,
    },
    keywordstyle={[4]\color{olive}},
    morekeywords={[4]
        forall,
        or,
        and,
        not,
        when,
        =,
        exists,
        at start,
        at end,
        over all,
        imply,
        oneof,
        increase,
    },
    comment=[l]{\;},
    sensitive=true
}[keywords, comments]

\lstdefinelanguage[mapl]{pddl}[]{pddl}
{
    morekeywords={[2]
      sensor,
    },
    morekeywords={[3]
      sense,
      replan
    },
    morekeywords={[4]
      KIF,
    },
}

\newcommand{\pddl}[1]{\text{\lstinline[basicstyle=\footnotesize\ttfamily]{#1}}}

\usepackage{color}

\begin{document}
%
\title{Fairness in Multi-Agent Planning}
\author{Alberto Pozanco, Daniel Borrajo\thanks{On leave from Universidad Carlos III de Madrid}
}
\affiliations{
J.P. Morgan AI Research \\
\{alberto.pozancolancho, daniel.borrajo\}@jpmorgan.com
}
\maketitle

\begin{abstract}
  In cooperative Multi-Agent Planning (MAP), a set of goals has to be achieved by a set of agents.
Independently of whether they perform a pre-assignment of goals to agents or they directly search for a solution without any goal assignment, most previous works did not focus on a fair distribution/achievement of goals by agents.
This paper adapts well-known fairness schemes to MAP, and introduces two novel approaches to generate cost-aware fair plans.
The first one solves an optimization problem to pre-assign goals to agents, and then solves a centralized MAP task using that assignment.
The second one consists of a planning-based compilation that allows solving the joint problem of goal assignment and planning while taking into account the given fairness scheme.
Empirical results in several standard MAP benchmarks show that
these approaches outperform different baselines. They also show that there is no need to sacrifice much plan cost to generate fair plans.
\end{abstract}

\section{Introduction}
Cooperative multi-agent planning (MAP)~\cite{DBLP:journals/csur/TorrenoOKS17} focuses on planning tasks where a set of agents need to achieve a set of goals.
These scenarios can be analyzed and solved using either distributed approaches, where the challenges lie in the communication among agents~\cite{DBLP:conf/aips/MaliahBS17} and privacy preservation~\cite{DBLP:conf/ijcai/Brafman15,DBLP:conf/aips/StolbaFK19}; or  centralized approaches, where the focus tend to be more on how problems are factored~\cite{DBLP:conf/aips/CrosbyRP13}.

In this paper we focus on the latter, and specifically on the combined problem of goal allocation and planning for multiple agents, i.e., how labor is distributed across agents and how these agents achieve the goals.
Previous works already addressed this joint problem~\cite{DBLP:journals/corr/MaK16,DBLP:conf/aips/0001D22}, but they (1) focus on multi-agent path finding rather than domain-independent MAP; and (2) focus on optimizing only one metric (typically makespan or plan length) without studying the implications of the induced goal assignments. 
While plan quality and privacy have been usually analyzed in the context of MAP tasks, fairness has not yet explicitly studied in the MAP literature as a metric to be taken into account.

\begin{figure}[hbt]
    \centering
    \includegraphics[scale=0.65]{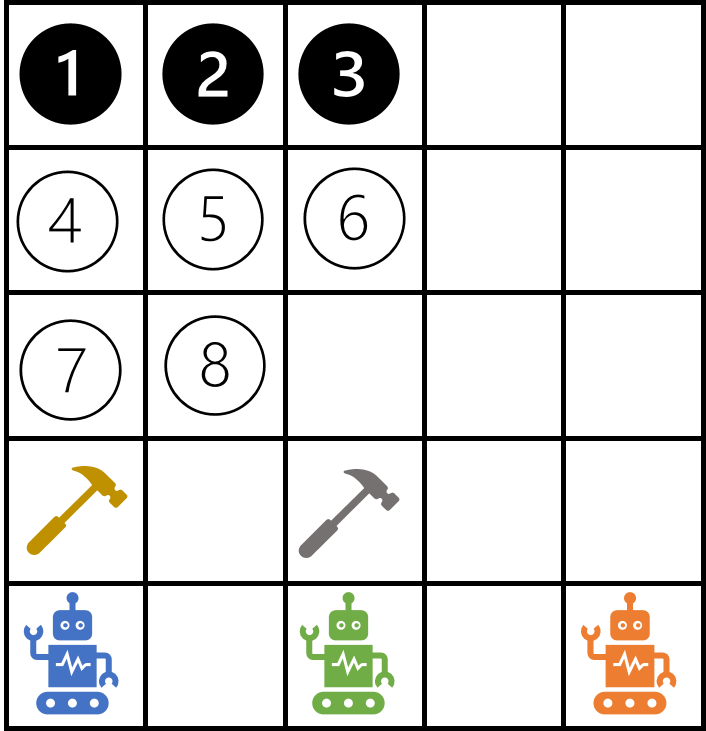}
    \caption{MAP task where three robots need to perform some works at eight different locations in a large warehouse.}
    \label{fig:fairness_example}
\end{figure}

Consider the MAP task depicted in Figure~\ref{fig:fairness_example}, where three robots controlled by a central entity (company) need to conduct some works at eight different locations in a large warehouse.
Robots can move over the grid (warehouse) in the four cardinal directions, pick up hammers and perform works. 
Moving consumes one energy unit and time, while picking up hammers and performing works do not consume energy or time, as the time needed to perform them is negligible compared to moving in the warehouse.
Robots need to pick up one of the two hammers in order to perform the works in black-filled locations (1, 2 and 3), while they do not need the hammer to perform works in the other locations.
If the company wants to minimize energy consumption, an optimal plan could involve the blue robot picking up the golden hammer and conducting the works in all locations in the following order: 7, 4, 1, 2, 3, 6, 5, and 8.
Such plan would have a total energy cost (sequential plan length) of $9$.
These are the type of plans returned by most domain-independent solvers that optimize sequential plan length~\cite{DBLP:journals/jair/Helmert06}.
The company might be also interested in finishing all the works as soon as possible.
In this case, an optimal plan could involve the blue robot picking up the golden hammer and performing works in locations 7, 8, 5, 4, and 1; and the green robot picking up the grey hammer and performing works in locations 6, 5, 2, and 3.
Such plan would allow the works to be completed in $6$ time steps (parallel plan length).
These are the type of plans returned by domain-independent solvers that optimize parallel plan length~\cite{rintanen2014madagascar}.
Note that in both plans the orange robot is not assigned to perform any work.
Even if these are optimal solutions under the different objectives, there exist a number of reasons why the company might not be interested in under-use of its resources. 
For example, in this case the company might be interested in balancing the workload to minimize robots' deterioration or maximize plan's robustness, i.e., maximize the number of goals achieved in case any of the robots fail.
Fair plans are even more important when humans are the planning task agents, since companies will be interested in balancing the workload to prevent employees' burnout; or will like to avoid employees' requests for explanations~\cite{DBLP:journals/corr/abs-1709-10256} so as to why peers with similar skill sets receive different number of goals/tasks.

A common solution to solve this problem is to divide the planning task by: (1) first assigning the goals to the agents in order to optimize a given metric, i.e., balance the load, minimize the cost etc.; and (2) then computing a plan where agents achieve their assigned goals.
However, by solving the planning task separately, we do not have any guarantee about the quality (in terms of plan cost or length) of the generated plans, or even about the existence of a plan that satisfies the assignment.
Moreover, most previous works do not reason about fair goal allocation explicitly: how different fairness notions generate different plans, and the trade-offs involved in having fair plans versus just optimizing the plan cost (length).
Fairness here is different from fairness as considered in Fully Observable Non-Deterministic (FOND)
 planning~\cite{DBLP:conf/aips/RodriguezBSG21}, where fairness is related to actions and the assumptions made about their non-deterministic effects.

The main contributions of this paper are:
\begin{enumerate}
    \item Adaptation of two well-known fairness definitions to domain-independent MAP.

    \item Two different approaches to solve a MAP task in a fair way: (1) a fairness-driven optimization method to pre-assign goals to agents; and (2) a planning-based compilation that allows solving the joint problem of goal assignment and planning while taking into account different fairness notions.
    
    \item Evaluation of the proposed
    approaches in domains from the last Competition of Distributed and Multi-Agent Planners (CoDMAP)~\cite{DBLP:conf/aaai/StolbaKK16}. We compare the results obtained by our techniques against several baselines in order to analyze the fairness vs optimality trade-off in these planning tasks.
\end{enumerate}

\section{Background}
In this section, we will cover the background on Automated planning, both in the single-agent and multi-agent settings and the standard language to represent planning tasks.

\subsection{Automated Planning}
Automated Planning (AP) is the task of choosing and organizing a sequence of actions such that, when applied in a given
initial state, it results in a goal state~\cite{DBLP:books/daglib/0014222}.
We use the first-order (lifted) planning formalism~\cite{DBLP:conf/aips/HollerB22}, where a \textit{single-agent classical planning task} is a tuple $\Pi = \langle \Psi, O, {\cal T}, \bar{A}, I, G \rangle$, where $\Psi$ is a set of \textit{predicates}; $O$ is a set of \textit{objects}; $\cal T$ is a set of \textit{object types}; $\bar{A}$ is a set of \textit{action schemas}; $I$ is the \textit{initial state} and $G$ is the \textit{goal state}.

Each predicate $P^n\in \Psi$ might have a tuple of parameters $P(v_1,\ldots,v_n)$ where $n$ is its arity.
Each variable is associated with a type (e.g. $v_i^{t_i}$, $t_i\in\cal T$).
$\mbox{par}(P)$ is the set of parameters, $P[i]$ is the $i$-th parameter of $P$, and $P_s$ is the \textit{predicate symbol} of $P$, i.e., its name.

Each object has an associated type. Let $O^t$ be the subset of objects $O$ of type $t \in {\cal T}$.
We assume a type hierarchy that is a tree in which $O^{tp} \supseteq O^{tc}$ holds, for a parent type $tp$ and a child type $tc$, .
A parameter $v_i^{t_i}$ can be substituted by an object from $O^{t_i}$.

When all the variables of a predicate are substituted, we obtain a \textit{ground atom}.
We use the same notation as in the case of predicates to access ground atoms.
We denote the set of all ground atoms or propositions of a planning task as $F$.
A \textit{state}, such as $I\subseteq F$ or $G \subseteq F$, consists of a set of ground atoms, interpreted as the set of all true ground atoms in that state $s$.
We will use $\bar{s}$ to refer to the set of lifted predicates of the atoms in some list of ground atoms $s\subseteq F$.

Action schemas $\bar{a} \in \bar{A}$ are tuples $\bar{a} = \langle \mbox{name}(\bar{a}), \mbox{par}(\bar{a}), \mbox{pre}(\bar{a}), \mbox{eff}(\bar{a}), c(\bar{a}) \rangle$ defining the action name, parameters $\bar{a}(v_1^{t_1}, \ldots, v_n^{t_n})$, preconditions, effects, and cost, respectively.
Similarly to predicates, we access the $i$-th parameter of a lifted action as $\bar{a}[i]$.
We assume $\mbox{eff}(\bar{a}) = \{ \mbox{add}(\bar{a}) \cup \mbox{del}(\bar{a}) \cup \mbox{cond}(\bar{a}) \}$ where the first two sets are the sets of add and delete predicates, and the third one is a set of lifted conditional effects.
Each lifted conditional effect in $\mbox{cond}(\bar{a})$ is of the form $C \rhd E$ and is composed of two sets of predicates: $C \subseteq \Psi$, the condition, and $E \subseteq \Psi$, the effect, which is a conjunction of add and del effects.
Finally, we assume action schemas are well-defined, i.e., the type of the action's variables must be a subtype (mapping to the same or a subset of  objects) of the variable types of the predicates in the action's pre and eff.

Similar to predicates, a \textit{ground action} (or just \textit{action}) is obtained by substituting each of its parameter variables $v_i^{t_i}$ by an object in $O^{t_i}$.
We use $A$ to denote the set of all ground actions.
Each action $a \in A$ will be composed of the action name, ground preconditions (pre(a)), ground effects (add/del/cond(a)) and cost, c(a).

The execution of an action $a\in A$ in a state $s$ is defined by a function $\gamma$ such that $\gamma(s,a)=(s\setminus\mbox{del}(a) \cup \mbox{del}_c(s,a))\cup\mbox{add}(a) \cup \mbox{add}_c(s,a)$ if pre($a$)$\subseteq s$, and $s$ otherwise (it cannot be
applied).
The conditional add, $\mbox{add}_c(s,a)$, and delete, $\mbox{del}_c(s,a)$, effects are the effects whose conditions hold in $s$:
\[
\text{add}_c/\text{del}_c(s,a) = \bigcup_{C \rhd E \in cond(a), C \subseteq s, e\in add/del(E)} \{e\}
\]

The output of a planning task is a sequence of actions, called a plan, $\pi=(a_1,\ldots,a_n)$. The execution
of a plan $\pi$ in a state $s$ can be defined as:
\begin{small}
\[\Gamma(s,\pi)=\left\{\begin{array}{ll}
                        \Gamma(\gamma(s,a_1),(a_2,\ldots,a_n)) & \mbox{if } \pi\neq \emptyset\\
                        s & \mbox{if } \pi=\emptyset\\
                      \end{array}
                    \right.
\]
\end{small}

A plan $\pi$ is valid if $G\subseteq\Gamma(I,\pi)$. The plan cost is commonly defined as
$c(\pi)=\sum_{a_i\in\pi} c(a_i)$. 
A plan with minimal cost is called optimal.

\subsection{Multi-Agent Planning}
Multi-agent planning (MAP) aims at solving problems in which several agents act in the same environment. 
We consider centralized cooperative MAPs in which a centralized entity with all the information plans for all the agents~\cite{DBLP:journals/csur/TorrenoOKS17}.
\begin{definition}
  A \textbf{MAP task} is a tuple
  \[{\cal P} = \langle N, \Psi, O, {\cal T}, \bar{\cal A},I,G\rangle,\]
  where: $N = \{1,\ldots,n\}$ is a set of agents; $\Psi$ is a set of predicates; $O$ is a set of objects, which includes the agents in $N$; $\cal T$ is a set of types, which include a hierarchy of agents' types; $\bar{\cal A} = \bigcup_{\forall i \in N}\bar{A}_i$ is the total set of action schemas of $\cal P$; $I$ is an initial state; and $G$ is the common goal state.
\end{definition}

We assume actions are executed by only one agent, i.e., there is no joint action where more than one agent is needed to carry out an action;
actions $\bar{a} \in \bar{A_i}$ have a $v_1^{t_i}$ variable in the first position of their parameters, indicating agents of type $t_i$ can execute action $\bar{a}$.
We also assume each action $\bar{a}$ can achieve at most one goal predicate $\bar{g} \in \bar{G}$.

MAP tasks grounding, action and plan execution are equivalent to the single-agent classical planning case.
The solution to a MAP task in our setting is a sequential plan $\pi$ that specifies the actions agents in $N$ need to execute so that all the common goals $G$ are achieved.
We refer as $\pi_i$ to the subset of actions in the plan executed by agent $i \in N$.

\subsection{The Planning Domain Definition Language}
The planning community uses a standard language to specify planning tasks in a domain-independent fashion: the Planning Domain Definition Language (PDDL)~\cite{mcdermott1998pddl}. 
PDDL separates the task definition into two parts: domain and problem.
The domain describes the object types, predicates and actions that can be used in a domain to solve a task, while the problem defines the objects, initial state and goals of the specific task to be solved.
Listing~\ref{perform} shows the lifted representation in PDDL of the \pddl{perform-work-black-location} action in our warehouse domain.
As we can see, the action has three parameters (or variables), denoted with symbols starting with question marks: \pddl{?a}, of type agent; \pddl{?l}, of type black location; and \pddl{?h}, of type hammer.
In order to execute the action in a given state, two preconditions need to hold: agent \pddl{?a} needs to be at location \pddl{?l}, and the same agent \pddl{?a} needs to be holding a hammer \pddl{?h}.
The effect of executing an action is that the work at that location has been performed, i.e., the predicate \pddl{(work-performed ?l)} is now true.
\begin{center}
\begin{lstlisting}[language=pddl,frame=single,caption=PDDL representation of the \pddl{perform-work-black-location} action in the warehouse domain.,label=perform]
(:action perform-work-black-location
 :parameters (?a - agent ?l - black_location ?h - hammer)
 :precondition (and (at ?a ?l) (holding ?a ?h))
 :effect (work-performed ?l))
\end{lstlisting}
\end{center}

Lifted actions are grounded by substituting each of the parameters by objects of their type.
If we ground and execute this action with \pddl{blueRobot} as agent, \pddl{L1} as location, and \pddl{goldenHammer} as hammer, the grounded atom \pddl{(work-performed L1)} would be true.

\section{Plan Fairness}
In this section we define four different plan fairness notions $\cal F$, which are adaptations of the well-known Maximin and Proportional Equality fairness schemes~\cite{DBLP:journals/ior/BertsimasFT11,brandt2016handbook} to MAP. 

The first two plan fairness notions we define are related to the number of goals each agent \emph{achieves} once the plan is executed.
We say an agent achieves a goal when it is the first agent making that goal proposition true.
Formally:
\begin{definition}
Given a plan $\pi$ that solves a MAP task $\cal P$, an \textbf{agent} $i \in N$ \textbf{is the first achiever of a goal} $g \in G$ iff it is the first agent in making  $g$ true. We denote $g_i$ when agent $i$ is the first achiever of goal $g$.
\end{definition}

We now formally define two plan fairness schemes that are \emph{goal-oriented}, i.e., they reason about the number of goals achieved by each agent.
\begin{definition}
 A \textbf{plan} $\pi$ is \textbf{optimal} under the \textbf{Goal Maximin} fairness scheme iff its execution maximizes the minimum number of goals from $G$ first achieved by any agent:
    \begin{gather*}
        \max(\min_{i\in N} |g_i|)
    \end{gather*}
\end{definition}

\begin{definition}
 A \textbf{plan} $\pi$ is \textbf{optimal} under the \textbf{Goal Proportional Equality} fairness scheme iff its execution minimizes the difference between the agent $i\in N$ that first achieves more goals from $G$ and the agent $k \in N$ that first achieves less: 
    \begin{gather*}
        \min (\max_N |g_i| - \min_N |g_k|)
    \end{gather*}
\end{definition}

We also define two plan fairness schemes that are \emph{workload-oriented}, i.e., they reason about the cost of the actions executed by each agent.

\begin{definition}
 A \textbf{plan} $\pi$ is \textbf{optimal} under the \textbf{Workload Maximin} fairness scheme iff its execution maximizes the minimum cost of the actions executed by any agent:
    \begin{gather*}
        \max(\min_{i\in N} c(\pi_i))
    \end{gather*}
\end{definition}

\begin{definition}
 A \textbf{plan} $\pi$ is \textbf{optimal} under the \textbf{Workload Proportional Equality} fairness scheme iff its execution minimizes the difference between the agent $i\in N$ with the most costly actions and the agent $k \in N$ with the least costly actions:
    \begin{gather*}
        \min (\max_N c(\pi_i) - \min_N c(\pi_k))
    \end{gather*}
\end{definition}

Let us exemplify these plan fairness notions using again the MAP task depicted in Figure~\ref{fig:fairness_example}.
A plan that would optimize the Goal Maximin fairness scheme could involve the blue robot performing the works in locations 1, 2, 4 and 7; the green robot performing the works in locations 5, and 8; and the orange robot performing the works in locations 3 and 6. 
This plan would however not be optimal under the Goal Proportional Equality plan fairness scheme.
In that case an optimal plan could involve the blue and green robots performing works in three locations, and the orange robot performing works in two locations. 
These plans are different not only between them, but also with respect to the standard definitions of sequential and parallel plan length we discussed in the introduction.

Discussing the workload version of the fairness schemes becomes trickier, as now we need to reason about the plans and not only about the goal allocation induced by a plan.
In this case, arbitrarily bad plans where agents perform irrelevant or loopy actions to increase their workload without contributing towards achieving any goal could optimize Workload Maximin or Workload Proportional Equality.
In the next section we discuss how we overcome this by generating cost-aware fair plans.

\section{Generating Cost-Aware Fair Plans}
As we have seen, arbitrarily bad plans could optimize the workload version of the plan fairness schemes.
Something similar could happen with their goal version, where the aimed goal distributions can be often accomplished by plans where agents might follow highly suboptimal paths to achieve their goals.
Hence, we will jointly optimize plan fairness and plan cost in a multi-objective way.

We present two different approaches to solve MAP tasks jointly considering plan fairness and cost. 
The first approach assigns goals to agents before planning, and then solves a MAP task with the chosen assignment. 
This approach can optimize the four different fairness schemes by slightly modifying the objective function that pre-assigns goals to agents.
The second approach compiles the MAP task in such a way that the resulting planning task considers both metrics while planning.
As discussed later, this approach can only optimize the two goal-oriented fairness schemes.

\subsection{Pre-assignment of Goals Based on Fairness}
This approach pre-assigns each goal to an agent, and then solves a centralized MAP task where goals have to be achieved by their assigned agents.
However, some goal propositions might be true in the initial state.
Since the plan might not need actions to re-achieve them, we do not assign those goals to any agent. In any case, if they have to be re-achieved by any agent, they will be considered for the plan fairness computation. 
Hence, we will just allocate assignable goals when performing goal pre-assignment.
\begin{definition}
A \textbf{goal} $g \in G$ \textbf{is assignable} iff $g \notin I$. We refer as $\hat{G} \subseteq G$ to the set of assignable goals.
\end{definition}
Let us consider a slight variation of our running example, where now robots can also drop hammers.
In this case as part of the goals, robots need to leave hammers in their original locations.
This is already true in the initial state, so we cannot assign these goals to any agent. 
However, an agent will achieve these goals eventually in order to generate a valid plan for the MAP task, and they will be considered for the plan fairness computation.

The first step of this algorithm is to solve a Mixed-Integer Linear Program (MILP) to assign assignable goals to agents.
The objective function to optimize will be: (1) when ${\cal F}=\mbox{Goal Maximin}$; (2) when ${\cal F}=\mbox{Goal Proportional Equality}$; (3) when ${\cal F}=\mbox{Workload Maximin}$; and (4) when ${\cal F}=\mbox{Workload Proportional Equality}$\footnote{Note that in objective functions 2 and 4 $\max-\min$ has been exchanged for $\min-\max$ to convert the minimization task into a maximization one.}:
    \begin{equation}
    \small
    \label{objective_function_maximin}
    \text{maximize}\quad \displaystyle {\cal K} \mbox{minG} - \!\!\!\!\sum_{(a,g) \in \mbox{achievable}(N,\hat{G})} \!\!\!\!\!\!\!\!\!\!\!\!\!x_{a,g} h_{\textsc{ff}}(\{a\},\{g\})
    \end{equation}
    \begin{equation}
    \small
    \label{objective_function_propeq}
    \text{maximize}\quad \displaystyle {\cal K} (\mbox{minG}-\mbox{maxG}) - \!\!\!\!\!\!\!\!\!\!\!\!\!\sum_{(a,g) \in \mbox{achievable}(N,\hat{G})} \!\!\!\!\!\!\!\!\!\!\!\!\!x_{a,g} h_{\textsc{ff}}(\{a\},\{g\})
    \end{equation}
    \begin{equation}
    \small
    \label{objective_function_maximin_G}
    \text{maximize}\quad \displaystyle {\cal K} \mbox{minW} - \!\!\!\!\sum_{(a,g) \in \mbox{achievable}(N,\hat{G})} \!\!\!\!\!\!\!\!\!\!\!\!\!x_{a,g} h_{\textsc{ff}}(\{a\},\{g\})
    \end{equation}
    \begin{equation}
    \small
    \label{objective_function_propeq2}
    \text{maximize}\quad \displaystyle {\cal K} (\mbox{minW}-\mbox{maxW}) - \!\!\!\!\!\!\!\!\!\!\!\!\!\sum_{(a,g) \in \mbox{achievable}(N,\hat{G})} \!\!\!\!\!\!\!\!\!\!\!\!\!x_{a,g} h_{\textsc{ff}}(\{a\},\{g\})
    \end{equation}
    subject to the following constraints:
    \begin{equation}
    \small
    \label{goal_assigned_to_only_one_agent}
    \sum_{(a,g) \in \mbox{achievable}(N,\{g\})}  x_{a,g}  = 1 , \  g \in \hat{G}
    \end{equation}
    \begin{equation}
    \small
    \label{min_equal_to_min_sum_of_goals_agent}
    \sum_{(a,g) \in \mbox{achievable}(\{a\},\hat{G})}  x_{a,g}  \geq \mbox{minG} , \  a \in N
    \end{equation}
    \begin{equation}
    \small
    \label{max_equal_to_max_sum_of_goals_agent}
    \sum_{(a,g) \in \mbox{achievable}(\{a\},\hat{G})}  x_{a,g}  \leq \mbox{maxG} , \  a \in N
    \end{equation}
    \begin{equation}
    \small
    \label{min_equal_to_min_sum_of_goals_agent_workload}
    \sum_{(a,g) \in \mbox{achievable}(\{a\},\hat{G})}  x_{a,g} h_{\textsc{ff}}(\{a\},\{g\})   \geq \mbox{minW} , \  a \in N
    \end{equation}
    \begin{equation}
    \small
    \label{max_equal_to_max_sum_of_goals_agent2}
    \sum_{(a,g) \in \mbox{achievable}(\{a\},\hat{G})}  x_{a,g} h_{\textsc{ff}}(\{a\},\{g\})  \leq \mbox{maxW} , \  a \in N
    \end{equation}
    
    We create one binary decision variable $x_{a,g}$ for each agent $a \in N$ and goal $g \in \hat{G}$.
    We also create four integer decision variables: $\mbox{minG}$ and $\mbox{maxG}$, which will keep track of the agent with most and least assigned goals; and $\mbox{minW}$ and $\mbox{maxW}$, which will keep track of the agent with highest and lowest estimated cost to achieve the goals.
    The MILP first optimizes the given fairness scheme function depending on $\cal F$(Eq. 1 to 4).
    This is done by multiplying the given $\min$ or $\min-\max$ integer variables by a constant $\cal K$ that defines how important maximizing the fairness value is. Since, in this paper we would like to give a high priority to fairness, we use a very large value for $\cal K$.
    The values of the $\min$ and $\max$ integer variables are ensured by Constraints 6 to 9.
    Constraints 8 and 9 are not needed when optimizing for goal-oriented fairness schemes, and Constraints 6 and 7 are not needed when optimizing for workload-oriented fairness schemes.
     Likewise, Constraints 7 and 9 are only needed when optimizing for any kind of Proportional Equality, since Maximin does not need to reason about the $\max$ number of goals / actions cost of any agent.
    As a secondary objective, the MILP tries to assign goals in $\hat{G}$ to agents that can achieve them with least cost.
    This is done by multiplying the binary decision variable $x_{a,g}$ ($1$ if goal $g$ is assigned to agent $a$, $0$ otherwise) by the \textsc{ff}~\cite{DBLP:journals/jair/HoffmannN01} heuristic value of $a$ achieving $g$ from $I$: $h_{\textsc{ff}}(\{a\},\{g\})$. 
    Constraint~\ref{goal_assigned_to_only_one_agent} ensures that each goal is assigned to only one agent.
    Finally, we ensure goals cannot be assigned to agents that cannot achieve them by using the \textit{achievable} function. Given a set of agents and a set of goals, it returns a set of pairs (agent $a$, goal $g$) such that $h_{\textsc{ff}}(\{a\}, \{g\})\not=\infty$.
    The solution to the MILP is \textsc{assignment}, a set containing  $(i \in N,g \in \hat{G})$ tuples.
    
    After the assignment is computed, a new MAP task is created from $\cal P$, where assignable goals are associated to specific agents as follows.
    We update the predicates in $\Psi$:
    \begin{itemize}
        \item $\bar{AGP} = \bigcup_{\bar{g} \in \bar{\hat{G}}} \{\bar{g}_s\mbox{-}ab(\mbox{par}(\hat{g}) \cup \{v_i^{agent}\})\}$, which is a set of $\bar{g}\mbox{-}ab$ predicates that augment the assignable goal predicates $\bar{\hat{G}}$ with the agent achieving them. 
        We use $AGP$ to denote the ground propositions of $\bar{AGP}$ substituted using the tuples in \textsc{assignment}.

        \item $d_g()$, which is a predicate symbol indicating that the goal proposition $g \in \hat{G}$ has already been achieved by an agent.
        This ensures we only attribute the achievement of a goal to the first agent in making it true (Def. 3).
    \end{itemize}
    Going back to our running example, having a goal to perform a work at location 1, \pddl{(work-performed L1)}, we would generate a predicate \pddl{(work-performed-AB ?location ?agent)} and a predicate \pddl{work-performed-L1-done}().
    
    Then, we update the actions that can achieve the goals.
    The function $\textsc{effPredGoal}(\bar{a}, g)$  returns the predicate in $\mbox{eff}(\bar{a})$ that, when grounded with the action's parameters $\mbox{par}(\bar{a})$ using the objects in $O$ and the types in $\cal T$, is equal to the assignable goal $g \in \hat{G}$.
    We assume predicates returned by this function do not include agents in their parameters, i.e., these predicates do not have an associated agent a priori.
    If such a predicate does not exist in $\mbox{eff}(\bar{a})$, \textsc{effPredGoal} returns $\emptyset$ and $\bar{a}$ is not updated.
    \begin{itemize}
        \item $\bar{{\cal A}}_{\neg u} = \{ \bar{a}\in \bar{{\cal A}}: \bigcup_{g \in \hat{G}}\{\textsc{effPredGoal}(\bar{a},g) = \emptyset\} \}$ is the subset of original action schemes $\bar{{\cal A}}$ that cannot achieve goals. 
        We do not update or modify these actions.

    \item $\bar{{\cal A}}_{u} = \{ \bar{u_{a}}: \bar{a} \in \bar{{\cal A}}, \bigcup_{g \in \hat{G}} \{\textsc{effPredGoal}(\bar{a},g) \neq \emptyset\}\}$ is the subset of $\bar{{\cal A}}$ that can achieve goals, updated with conditional effects that account for the agent achieving each goal proposition. As a reminder, we assume agents are the first parameter of an action ($\bar{a}[1]$).
    \begin{itemize}
    \item $\mbox{eff}(\bar{u_{a}}) = \mbox{eff}(\bar{a}) \cup  $ \begin{multline}
        \bigcup_{g \in \hat{G}} \Big\{\{\neg d_g()\} \!\!\!\!\!\!\!\!\!\!\!\! \bigcup_{\forall i =1..n-1, P^n= \textsc{effPredGoal}(\bar{a},g)} \!\!\!\!\!\!\!\!\!\!\!\! \{P[i] = g[i]\}\Big\} \rhd \\
        \Big\{d_g(), g_s\mbox{-}ab(\mbox{par}(P) \cup \{\bar{a}[1]\}) \Big\}
     \end{multline}
    Parameters, preconditions and cost of each $\bar{u_{a}}$ are equal to their $\bar{a}$ counterpart.
    \end{itemize}
    \end{itemize}
    
    The new planning task is defined as follows:
    \begin{definition}
    Given an original MAP task ${\cal P}$, a \textbf{labeled fair MAP task} is a tuple 
  \[{\cal P}_{{\cal F}_L} = \langle N, \Psi^\prime, O, {\cal T}, \bar{\cal A}^\prime,I,G^\prime\rangle\]
   where,
   \begin{itemize}
       \item $\Psi^\prime = \Psi \cup \bar{AGP} \cup \{d_g()\}$
       \item $\bar{\cal A}^\prime =  \bar{\cal A}_{\neg u} \cup \bar{\cal A}_{u}$
       \item $G^\prime = G \cup AGP $
   \end{itemize}
    \end{definition}
    
    PDDL examples of the compilation of complete tasks by this approach can be found in Appendix A.
    Since we are performing goal assignment separated from planning, this compilation is not complete. 
    Let us go back to our warehouse running example. 
    If we optimize for the Goal Maximin fairness scheme, the following could be an optimal allocation of goals to agents: the blue robot performs works at locations 1, 4 and 7; the green robot performs works at locations 2, 5 and 8; and the orange robot performs works at locations 3 and 6.
    However, when trying to solve the equivalent labeled fair MAP task, the problem will not have a solution, since there are only two hammers and the three robots have been assigned to perform a work in a black-filled location.
    In this case, the only valid solutions require one robot to just performs tasks at white-filled locations.
    This compilation is not optimal for the same reason, as we are not considering the interactions between goals.
    
    As previous works have demonstrated in the context of multi-agent path finding~\cite{DBLP:journals/corr/MaK16,DBLP:conf/aips/0001D22}, this labeled MAP task where goals are assigned to agents is harder than the unlabeled task, where any agent can achieve any goal.

\subsection{Plan Fairness Compilation}
This approach consists of a planning-based compilation that allows us to jointly solve the goal allocation and planning problems while accounting for the given fairness scheme.
The intuition behind the compilation is to extend a standard MAP task with new predicates that keep the count of how many goals have been achieved by each agent. It also introduces new actions that provide a different \emph{reward} at the end of planning, depending on the induced goals distribution.
This compilation is possible for the goal-oriented fairness schemes, as the number of goals is known a priori, thus having a bounded number of distributions of goals to agents.
However, the distribution of actions cost to agents is not bounded, so in this paper we will focus on a compilation for the goal-oriented fairness schemes, leaving as future work devising a feasible compilation for the more complex workload-oriented fairness schemes.

Unless the previous compilation, this approach could handle all the goals and not only the subset of them not true in $I$.
This is because goals are indirectly assigned to agents during planning, so the planner knows which goals from $G$ have to be re-achieved.
However, we opt to formalize the compilation over $\hat{G}$ due to efficiency issues we will further discuss.

Given a MAP task ${\cal P}$, we extend ${\cal T}$ with a type $num$, that will help us to symbolically represent numbers. 
We also extend $\Psi$ with the following predicates:
\begin{itemize}
    \item $gc(i^{agent},n^{num})$ represents the number of goals $n \in O^{num}$ achieved by agent $i \in O^{agent}$. 
    
    \item $next(n^{num},m^{num})$ represents sequential ordering of numbers, i.e., $m \in O^{num}$ follows $n \in O^{num}$. 

    \item $d_g()$ is the same one defined in the previous approach.
    
    \item $end()$ represents all goals have been achieved.
\end{itemize}

We also extend the original set of action schemes $\bar{\cal A}$ with (1) updated original action schemes that keep the count of the number of goals achieved by each agent; and (2) action schemes that increase the total plan cost based on the value of these counters and the fairness scheme $\cal F$.
\begin{itemize}

    \item $\bar{{\cal A}}_{\neg u}$ is defined as in the previous approach.
    
    \item $\bar{{\cal A}}_{u} = \{ \bar{u_{a}}: \bar{a} \in \bar{{\cal A}}, \bigcup_{g \in \hat{G}} \textsc{effPredGoal}(\bar{a},g) \neq \emptyset \}$, is the subset of $\bar{{\cal A}}$ that can achieve goals, updated with conditional effects that account for the specific agent achieving each goal proposition.
    \begin{itemize}
        \item $\mbox{par}(\bar{u_{a}}) = \mbox{par}(\bar{{a}}) \cup \{n^{num}, m^{num}\}$ 
    
        \item $\mbox{pre}(\bar{u_{a}}) = \mbox{pre}(\bar{a}) \cup \{gc(i^{agent},n), next(n,m)\}$

        \item $\mbox{eff}(\bar{u_{a}}) = \mbox{eff}(\bar{a}) \cup  $ \begin{multline}
        \bigcup_{g \in \hat{G}} \Big\{\{\neg d_g()\} \!\!\!\!\!\!\!\!\!\!\!\! \bigcup_{\forall i =1..n-1, P^n= \textsc{effPredGoal}(\bar{a},g)} \!\!\!\!\!\!\!\!\!\!\!\! \{P[i] = g[i]\}\Big\} \rhd \\
        \Big\{d_g(), \neg gc(i,n), gc(i,m) \Big\}
     \end{multline}
        \item $c(\bar{u_{a}}) = c(\bar{a})$
    \end{itemize}
    We are updating the goal counter predicates $gc(i,n)$ according to Definition 2, i.e., the goal counter is updated only when the agent $i$ is the first achiever of the goal ($\neg d_g()$).
    A PDDL representation of an action in $\bar{{\cal A}}_u$ in our warehouse domain (see Listing~\ref{perform}) is shown in Listing~\ref{stack}. As we can see, the compilation augments the original preconditions with the current number of goals of the agent \pddl{(n_goal ?a ?n1)} and the next number \pddl(next ?n1 ?n2)). It also augments the original effect with one conditional effect for each goal proposition that can be achieved using that action, i.e., perform the works in the three black locations, \pddl{L1}, \pddl{L2}, and \pddl{L3}. Examples of complete tasks can be found in Appendix A.
    
    \item $\bar{{\cal A}_{r}} = \{ \bar{r_{p}}: p\in \mbox{partitions}(|\hat{G}|,|N|) \}$  is a set of reward action schemas that assign a different cost to having a final goal distribution $p$ over the agents. 
    We generate as many action schemas as restricted partitions~\cite{andrews1998theory} of $|\hat{G}|$ into exactly $|N|$ parts are.
    These restricted partitions are all the possible ways in which we can allocate $|\hat{G}|$ goals to $|N|$ agents.
    For example, if we have four goals and two agents, this function will generate three different partitions (and therefore actions): $\{ (4,0)$, $(3,1)$, $(2,2) \}$.
    We assume numbers in each partition are automatically translated as objects of type \textit{num}.
    Each action $\bar{r_{p}}$ is generated as follows:
    \begin{itemize}
        \item $\mbox{par}(\bar{r_p}) = \bigcup_{i \in N} \{v_i^{agent}\}$ 
        \item $\mbox{pre}(\bar{r_{p}}) = G \cup \bigcup_{g \in \hat{G}} \{d_g()\} \cup \bigcup_{n \in N} \{gc(\bar{r_p}[i], p[i])\} \cup \bigcup_{(i,k) \in \mbox{par}(\bar{r_p})} \{i \neq k\}$ 
        \item $\mbox{eff}(\bar{r_p}) = \{end()\}$
        \item $c(\bar{r_p}) = \omega({\cal F}, p)$ is a function that maps a cost to each possible partition, i.e., distribution of goals to agents based on the fairness scheme $\cal F$.
        Thus, we are able to linearize non-linear functions such as $\min$ or $\max$. 
        When ${\cal F}=\mbox{Maximin}$, $\omega({\cal F},p)=(|\hat{G}| - \min(p)) \times {\cal K}$, where $\min(p)$ is the lowest number in the partition $p$ and $\cal K$ is a large constant.
        Therefore, higher $min(p)$ values imply lower costs.
        In case ${\cal F}=\mbox{Maximin}$, $\omega({\cal F},p)=(\max(p) - \min(p)) \times {\cal K}$.
        Therefore, lower $\max(p) - \min(p)$ values imply lower costs.
    \end{itemize}
    The precondition of these actions includes $G$ and not only the goal propositions in $\hat{G}$, which means that both assignable and non-assignable goals need to be satisfied in order to make $end$ true. 
\end{itemize}

    \begin{center}
\begin{lstlisting}[language=pddl,frame=single,caption=PDDL representation of the compiled \pddl{perform-work-black-location} action in the warehouse domain.,label=stack]
(:action perform-work-black-location
 :parameters (?a - agent ?l - black_location 
              ?h - hammer ?n1 ?n2 - number)
 :precondition (and (at ?a ?l) (holding ?a ?h)
                    (n_goal ?a ?n1) (next ?n1 ?n2))
 :effect (and (work-performed ?l)
              (when (and (= ?l L1)
                        (not (work-performed-L1-done))) 
                    (and (not (n_goal ?a ?n1))
                         (n_goal ?a ?n2)
                         (work-performed-L1-done)))
              (when (and (= ?l L2)
                         (not (work-performed-L2-done))) 
                    (and (not (n_goal ?a ?n1))
                         (n_goal ?a ?n2)
                         (work-performed-L2-done)))
              (when (and (= ?l L3)
                         (not (work-performed-L3-done))) 
                    (and (not (n_goal ?a ?n1))
                         (n_goal ?a ?n2)
                         (work-performed-L3-done))))
\end{lstlisting}
\end{center}

The new planning task is defined as follows:
    \begin{definition}
    Given an original MAP task ${\cal P}$, a \textbf{compiled fair MAP task} is a tuple 
  \[{\cal P}_{{\cal F}_C} = \langle N, \Psi^\prime, O^\prime, {\cal T}^\prime, \bar{\cal A}^\prime,I^\prime,G^\prime\rangle\]
   where,
   \begin{itemize}
       \item $\Psi^\prime = \Psi \cup \{ gc(i^{agent},n^{num}), next(n^{num},m^{num}),
       d_g(), end\}$
       \item $O^\prime = O \cup \bigcup_{n = (0..|G|+1)} \{n^{num}$\}
       \item ${\cal T}^\prime = {\cal T} \cup \{num\}$
       \item $\bar{\cal A}^\prime = \bar{\cal A}_{\neg u} \cup \bar{\cal A}_{u} \cup \bar{{\cal A}_{r}}$
       \item $I^\prime = I \cup \bigcup_{m=( 0..|G|+1)} next(m\!\!\!-\!\!\!1,m) \cup \bigcup_{i \in O^{agent}} \{gc(i,0)\}$
       \item $G^\prime = \{end()\} $
   \end{itemize}
    \end{definition}

We proceed now to formally prove that the compilation is complete and sound, as well as to discuss what defines an optimal plan in  ${\cal P}_{{\cal F}_C}$.

\begin{theorem}[Completeness] 
If there exists a plan $\pi$ that solves a MAP task ${\cal P}$, there exists a corresponding plan $\pi_{{\cal F}_C}$ that solves ${\cal P}_{{\cal F}_C}$.
\end{theorem}

\begin{proof}
The single goal in ${\cal P}_{{\cal F}_C}$ is the artificial $end()$ proposition.
This proposition can only be achieved by applying any of the ${\cal A}_r$ actions, which require as preconditions that: (1) goals in $G$ are achieved; (2) all $d_g()$ propositions are true; and (3) all agents in $N$ have an associated $gc$ proposition.
Regarding (1), since none of the actions in ${\cal P}_{{\cal F}_C}$ is modifying the original predicates in $\cal P$'s actions (only reasoning about extended predicates), and $I^\prime$ is a superset of $I$,  $\pi_{{\cal F}_C}$ can also achieve $G$ from $I'$.
In the case of (2), $d_g()$'s propositions are made true by the ${\cal A}_u$ actions when an agent first achieves an assignable goal. 
Since all goals $g \in G$ can be achieved, and $\hat{G} \subseteq G$,
all the $d_g()$'s propositions can also be achieved.
Finally, regarding (3), the $gc$ propositions already exist in $I^\prime$, being only updated by the ${\cal A}_u$ actions.
Therefore, if a plan $\pi$ exists, there exist a corresponding plan $\pi_{{\cal F}_C}$.
\end{proof}

\begin{theorem}[Soundness]
Any plan $\pi_{{\cal F}_C}$ that solves ${\cal P}_{{\cal F}_C}$ induces a plan $\pi$ that solves ${\cal P}$.
\end{theorem}

\begin{proof}
Plans that solve ${\cal P}_{\cal F}$ are composed of the original (extended) actions ${\cal A}_u$ and ${\cal A}_{\neg u}$, always finishing with one ${\cal A}_r$ action that achieves the $end()$ proposition.
If we remove  ${\cal A}_r$  from $\pi_{{\cal F}_C}$, and translate all ${\cal A}_u$ to their original counterparts by removing extra parameters, preconditions, and conditional effects, 
we will have a plan $\pi$ that achieves the original goals in $G$, therefore solving ${\cal P}$.
\end{proof}

Regarding the optimality of a plan that solves ${\cal P}_{{\cal F}_C}$, we ensure costs returned by $\omega(\cdot,\cdot)$ are orders of magnitude bigger than the rest of the actions' costs by using a very large $\cal K$.
Therefore, a plan that optimally solves ${\cal P}_{{\cal F}_C}$ will first try to optimize the given fairness scheme $\cal F$, while also optimizing the cost of the original (extended) actions ${\cal A}_u$ and ${\cal A}_{\neg u}$.

Since this approach consists of a planning compilation, it would be able to provide a solution for the warehouse running example our previous approach could not solve.
Given that the compilation approach solves each problem taking into account the fairness-driven cost function, it does not pre-commit to a given assignment as the other approach, and can recover through search of bad assignment decisions.
In our running example, the planner would realize that it is not possible to assign works at black-filled locations to all robots with only two hammers, returning a cost-aware fair plan where one robot only performs works at white-filled locations.
However, this completeness comes at the expense of solving a harder planning task than the labeled MAP our previous approach solves (and therefore harder than the unlabeled problem).
In this case, we are reasoning about all the possible goal distributions on top of the original MAP task.

Finally, note that by just removing the $d_g()$ predicates, we would have a compilation able to reason over all the goals (in terms of fairness) and not only assignable ones. However, we experimentally observed a large performance decrease when using this approach due to scalability issues.
This means the $d_g()$ predicates are able to better drive the search, and this gain pays off the loss of reasoning over $\hat{G}$ rather than $G$.

\section{Evaluation}

\subsection{Experimental Setting}
We compare our plan fairness approaches against several baselines with the aim of showing that our proposed approaches outperform state of the art planners and goal pre-assignment techniques in generating cost-aware fair plans.
All techniques are summarized in Table~\ref{tab:summary_approaches}.

\begin{itemize}
    \item \textsc{lama}. 
    This approach runs the \textsc{Fast Downward}~\cite{DBLP:journals/jair/Helmert06} planner with the \textsc{lama} setting~\cite{DBLP:journals/jair/RichterW10} to solve the original MAP task $\cal P$, where goals are not pre-assigned and there is no fairness scheme considered.
    
    \item \textsc{madagascar}, which runs the SAT-based planner of the same name~\cite{rintanen2014madagascar} to solve the original MAP task $\cal P$.
    While \textsc{lama} optimizes sequential plan length, the version of \textsc{madagascar} used here optimizes parallel plan length, potentially involving more agents than a planner just optimizing plan length.
    
    \item \textsc{contract-net}. This approach was employed in the \textsc{cmap} family of planners~\cite{DBLP:journals/kais/BorrajoF19}, and is inspired by the well-known negotiation schemes in multi-agent literature~\cite{DBLP:journals/tc/Smith80a}.
    Achievable goals are allocated to agents before planning through bids.
    Each goal is assigned to the best bidding agent in an iterative process.
    The quality of a bid is determined by 
    the heuristic value computed by {\sc ff} relaxed plan heuristic of the agent achieving that goal plus all its previously assigned goals.
    Usually, achieving multiple goal propositions is more expensive than only achieving one, so \textsc{contract-net} tends to assigns goals to the best agents while also taking into account all previous assignments, trying to not overload agents. Thus, it indirectly tries to maximize fairness ($\approx$).
    After the assignment, a new labeled fair MAP task (Def. 6) is created and solved with \textsc{lama}.
    
    \item \textsc{milp}, which runs our fairness-based pre-assignment of goals. Each MILP task is solved using the PuLP Python library~\cite{mitchell2011pulp} and the CBC solver~\cite{forrest2005cbc}.
    After the assignment, a new labeled fair MAP task (Def. 6) is created and solved with \textsc{lama}.
    \textsc{milp-g-maximin}, \textsc{milp-g-propeq}, \textsc{milp-w-maximin} or \textsc{milp-w-maximin} refer to the chosen fairness metric.
    
    \item \textsc{fpc}. Our fair planning compilation (Def. 9), which jointly solves the goal allocation and planning problem in a cost-sensitive manner. 
    \textsc{lama} is used to solve the compiled planning task.
    We will use \textsc{fpc-g-maximin} or \textsc{fpc-g-propeq} depending on the fairness metric optimized.
    
\end{itemize}

\begin{table}[hbt]
    \centering
    \begin{tabular}{c|c|c}
        Approach & Fairness & Goal Assignment   \\ \hline
        
        \textsc{lama} & \xmark & \xmark \\
        
        \textsc{madagascar} &  \xmark  & \xmark \\ \hline
        
        \textsc{contract-net} &  $\approx$ & Contract net\\
        
        \textsc{milp} &  \cmark &  MILP \\ \hline
        
        \textsc{fpc} &  \cmark & Planning
    \end{tabular}
    \caption{Summary of compared approaches.} 
    \label{tab:summary_approaches}
\end{table}

\begin{table*}[hbt]
\scriptsize
\setlength{\tabcolsep}{4pt}
    \centering
    \begin{tabular}{|c|c|c|c|r|r|r|r|r|r|r|}
    \hline
\textbf{Problems} & $\mathbf{|N|}$ & $\mathbf{|G|}$ &\textbf{Approach} & \textbf{Plan Cost} & \textbf{G-Maximin} & \textbf{G-Propeq} & \textbf{W-Maximin}& \textbf{W-Propeq} & \textbf{Total Time} & \textbf{Coverage}\\ \hline
\multirow{9}{*}{All Problems}& \multirow{9}{*}{$4.41\pm1.32$}  & \multirow{9}{*}{$12.54\pm6.67$} & \textsc{lama}& $\mathbf{167.19}$ & $67.0$ & $72.76$ & $36.61$ & $66.95$ & $244.06$ & $\mathbf{175}$ \\
& & & \textsc{madagascar}& $91.96$ & $104.88$ & $90.06$ & $\mathbf{138.48}$ & $81.52$ & $\mathbf{309.89}$ & $146$ \\
& & & \textsc{contract-net}& $137.53$ & $136.0$ & $122.6$ & $98.21$ & $99.21$ & $258.96$ & $171$ \\
& & & \textsc{milp-g-maximin}& $138.89$ & $142.52$ & $122.85$ & $99.68$ & $99.03$ & $256.34$ & $172$ \\
& & & \textsc{milp-g-propeq}& $136.0$ & $\mathbf{144.99}$ & $\mathbf{139.35}$ & $103.01$ & $105.16$ & $253.71$ & $170$ \\
& & & \textsc{milp-w-maximin}& $135.5$ & $142.23$ & $127.96$ & $108.78$ & $105.87$ & $250.85$ & $172$ \\
& & & \textsc{milp-w-propeq}& $134.37$ & $141.38$ & $132.42$ & $108.73$ & $\mathbf{106.34}$ & $253.94$ & $174$ \\
& & & \textsc{fpc-g-maximin}& $113.31$ & $128.45$ & $112.39$ & $51.73$ & $94.37$ & $190.04$ & $141$ \\
& & & \textsc{fpc-g-propeq}& $112.14$ & $127.37$ & $127.2$ & $51.71$ & $97.73$ & $197.74$ & $142$ \\
\hline

\multirow{9}{*}{\makecell{Commonly \\ Solved\\ Problems}}& \multirow{9}{*}{$3.7\pm0.93$}  & \multirow{9}{*}{$9.44\pm4.46$} & \textsc{lama}& $\mathbf{104.53}$ & $38.0$ & $47.43$ & $16.22$ & $38.91$ & $176.53$ & $110$ \\
& & & \textsc{madagascar}& $70.73$ & $78.32$ & $67.27$ & $\mathbf{105.63}$ & $55.43$ & $\mathbf{195.46}$ & $110$ \\
& & & \textsc{contract-net}& $87.55$ & $83.5$ & $78.14$ & $54.03$ & $59.73$ & $195.4$ & $110$ \\
& & & \textsc{milp-g-maximin}& $88.67$ & $91.12$ & $84.28$ & $55.04$ & $61.08$ & $191.56$ & $110$ \\
& & & \textsc{milp-g-propeq}& $87.16$ & $92.12$ & $93.05$ & $57.28$ & $63.16$ & $190.9$ & $110$ \\
& & & \textsc{milp-w-maximin}& $84.68$ & $89.31$ & $84.18$ & $62.67$ & $65.61$ & $186.03$ & $110$ \\
& & & \textsc{milp-w-propeq}& $84.22$ & $92.6$ & $89.81$ & $60.57$ & $62.22$ & $188.11$ & $110$ \\
& & & \textsc{fpc-g-maximin}& $89.66$ & $\mathbf{101.29}$ & $87.78$ & $35.05$ & $77.45$ & $157.15$ & $110$ \\
& & & \textsc{fpc-g-propeq}& $89.42$ & $100.2$ & $\mathbf{99.76}$ & $34.9$ & $\mathbf{78.95}$ & $164.87$ & $110$ \\
\hline
    \end{tabular}
    \caption{Score obtained by each approach on each metric for the two sets of problems. Bold numbers represent best performance.}
    \label{tab:fairness_summary_table}
\end{table*}

We evaluate these approaches in 9 domains and 180 problems taken from the last CoDMAP competition~\cite{DBLP:conf/aaai/StolbaKK16}.
We select tasks in the unfactored centralized track, and run the script provided by the competition organizers to translate the unfactored Multi-agent PDDL to plain PDDL plus a list of agents.\footnote{\url{http://agents.fel.cvut.cz/codmap/}}
We used all domains in the competition except for \textsc{taxi} and \textsc{wireless}, where goals are already assigned to agents; and \textsc{woodworking}, where actions achieve multiple goals.
On average, $92.4\%$ of the goals in these tasks are assignable, so \textsc{contract-net}, \textsc{milp} and \textsc{fpc} cannot reason in terms of fairness on less than a $8\%$ of the goals. 
Experiments were run on an Intel(R) Xeon(R) Platinum 8375 CPU @ 2.90GHz x 4 processors with a 24GB memory bound and a time limit of 900s.
Code and benchmarks are available upon request.

\subsection{Fairness Evaluation}
We measure the following metrics for each combination of task and solver.
For all metrics, the score of an approach on an unsolved task is 0. Otherwise, the score is computed as:

\textit{\textbf{Plan Cost}}: ratio $C^*/C$, where $C$ is the cost of the plan found by the given approach, and $C^*$ is the cost of the cheapest plan found by any approach on that task.
In the case of \textsc{fpc-maximin} and \textsc{fpc-propeq}, we remove the artificial cost added by the ${\cal A}_r$ actions by applying modulo $\omega$ to $c(\pi)$.

\textit{\textbf{G-Maximin}}: ratio $\mbox{G-Maximin}/\mbox{G-Maximin}^*$, where $\mbox{G-Maximin}$ is the lowest number of goals achieved by any agent in the plan found by the given approach, and $\mbox{G-Maximin}^*$ is the highest min number of goals achieved by any agent found by any approach on that task.

\textit{\textbf{G-Propeq}}: ratio $\mbox{G-Propeq}^*/\mbox{G-Propeq}$, where $\mbox{G-Propeq}$ and $\mbox{G-Propeq}^*$ are defined as in the case of $\mbox{G-Maximin}$, but using Def. 4 rather than Def. 3.

\textit{\textbf{W-Maximin}}: ratio $\mbox{W-Maximin}/\mbox{W-Maximin}^*$, where $\mbox{W-Maximin}$ is the min cost of the actions executed by any agent in the plan found by the given approach, and $\mbox{W-Maximin}^*$ is the highest min cost of the actions executed by any agent found by any approach on that task.

\textit{\textbf{W-Propeq}}: ratio $\mbox{W-Propeq}^*/\mbox{W-Propeq}$, where $\mbox{W-Propeq}$ and $\mbox{W-Propeq}^*$ are defined as in the case of $\mbox{W-Maximin}$, but using Def. 6 rather than Def. 5.

\textit{\textbf{Total Time}}: the score of an approach on a solved task is $1$ if the task was solved within $1$ second. 
    Otherwise, if the task was solved in $T$ seconds, 
    its score is $1 - log(T) / log(900)$.
    In the case of \textsc{contract-net} and the \textsc{milp} variants, $T$ also includes the time needed to perform the goal assignment. 
    
\textit{\textbf{Coverage}}: number of solved tasks.

The score of an approach for a domain is the sum of its scores for all the tasks in the domain, with higher scores meaning better performance.
The final distribution of goals to agents needed to compute the fairness metrics is determined by analyzing the output plan, updating the number of goals achieved by an agent according to Def. 2.

Table~\ref{tab:fairness_summary_table} summarizes the results 
in two sets of MAP tasks (per-domain results can be found in Appendix B).
The first set considers all the 180 MAP problems (20 per domain) without any fairness consideration.
As expected, \textsc{lama} obtains the best results in terms of plan cost and coverage, since it is solving the easier original MAP task.
However, these good plan-cost results come at the expense of getting the worst fairness results.
In fact, \textsc{lama} tends to use only one or two agents in most of the tasks across all domains, leaving the rest of agents idle.
\textsc{madagascar}, which is the fastest approach, outperforms \textsc{lama} in terms of fairness by using more agents when trying to optimize parallel plan length.
In fact, \textsc{madagascar} achieves the best results in the W-Maximin metric.
This is because it is unnecessarily making agents to execute more actions, as highlighted by the fact that it achieves the lowest plan-cost score.
\textsc{madagascar}'s fairness results are still far from those achieved by \textsc{contract-net} and the \textsc{milp} approaches.
These algorithms return plans that are $1.2$ times worse in cost than \textsc{lama}'s. But, their fairness scores are around $2.2$ times better than those of \textsc{lama}.
The \textsc{fpc} compilation approaches outperform those solving the original MAP tasks (\textsc{lama} and \textsc{madagascar}) in terms of fairness, but they cannot compete with algorithms pre-assigning goals due to scalability issues (lower coverage).
This is specially the case in domains involving many agents and goals such as \textsc{rovers}.

We consider a second set of MAP problems consisting on those tasks for which all the approaches can return a solution.
We do this to assess the potential benefit of using our planning compilation in tasks for which it can scale.
In this set of commonly solved problems, \textsc{fpc-g-maximin} and \textsc{fpc-g-propeq} achieve the best G-Maximin and G-Propeq results, respectively. 
Although the \textsc{fpc} approaches do not optimize for workload-oriented fairness, they are able to obtain compelling results in these metrics, specifically in W-Propeq, where \textsc{fpc-g-propeq} achieves the best results in the commonly solved problems.
This suggests that optimizing for a fair distribution of goals to agents tends to be a good proxy for optimizing a fair distribution of workload (cost).
\textsc{fpc} compilations also outperform all the approaches solving labeled MAP tasks in plan cost.
In fact, the plan cost score obtained by the \textsc{fpc} approaches is only $1.2$ times worse than the score obtained by \textsc{lama}, highlighting that our compilation does not need to sacrifice much plan cost in order to generate fair plans.

\section{Related Work}
ADP~\cite{DBLP:conf/aips/CrosbyRP13} and CMAP/MAPR~\cite{DBLP:journals/kais/BorrajoF19} perform pre-assignment of goals to agents.
None of these approaches reasons about different fairness metrics and their trade-offs explicitly as we do.
Moreover, goal assignment is performed separately from planning, while our proposed plan fairness compilation performs both at the same time.

Several works in robotics also perform pre-assignment of goals to agents~\cite{DBLP:journals/amai/Conitzer10,DBLP:journals/ijrr/GerkeyM04,vig2006multi}. 
Recent works in path planning simultaneously solve goal assignment and planning as we do~\cite{DBLP:journals/corr/MaK16,DBLP:conf/aips/0001D22,DBLP:conf/aips/AakashS22}, but they do not explicitly reason on how fairly goals are allocated to agents.
A key difference is that robotics and path-planning approaches are domain-dependent, while our approach is domain-independent.

Finally, fairness has been well studied in areas such as reinforcement learning~\cite{DBLP:conf/aaai/GrupenSL22}, machine learning~\cite{barocas2017fairness}, mechanism design~\cite{DBLP:conf/fat/FinocchiaroMMPR21} and operations research~\cite{ogryczak2014fair,karsu2015inequity}.
To the best of our knowledge, this is the first work considering fairness and its trade-offs in MAP.

\section{Conclusions and Future Work}
In this paper we argue that fairness is crucial in many MAP applications.
We have adapted two well-known fairness schemes (Maximin and Proportional Equality) to MAP, proposing four plan fairness schemes out of which two are goal-oriented, i.e., they reason about the number of goals achieved by each agent; and two are workload-oriented, i.e., they reason about the cost of the actions executed by each agent.
We have also introduced two novel approaches to generate cost-aware fair plans: (1) a pre-assignment of goals through a MILP that is then compiled into a labeled MAP task, which is able to scale to problems with many agents and goals; and (2) a compilation that solves the joint problem of goal allocation and planning, which outperforms other approaches in small to medium size problems.
Both techniques show that there is no need to sacrifice much plan cost to generate fair plans.

Currently our approaches optimize for the given fairness scheme, breaking ties in favor of better (i.e., less costly) plans.
In future work we would like to explore different weighting of these two objectives in order to generate Pareto fronts that humans can examine in order to decide what the right trade-off is between plan quality and fairness.

\section*{Acknowledgements}
This paper was prepared for informational purposes by
the Artificial Intelligence Research group of JPMorgan Chase \& Co. and its affiliates (``JP Morgan''),
and is not a product of the Research Department of JP Morgan.
JP Morgan makes no representation and warranty whatsoever and disclaims all liability,
for the completeness, accuracy or reliability of the information contained herein.
This document is not intended as investment research or investment advice, or a recommendation,
offer or solicitation for the purchase or sale of any security, financial instrument, financial product or service,
or to be used in any way for evaluating the merits of participating in any transaction,
and shall not constitute a solicitation under any jurisdiction or to any person,
if such solicitation under such jurisdiction or to such person would be unlawful.

\bibliography{main}
\newpage
\onecolumn

\section*{Appendix A }\label{AppendixA}

In this section we present an original MAP task ${\cal P}$ in the driverlog domain, as well as the two equivalent compilations we introduced in this paper: labeled fair MAP task, in which goals are pre-assigned (Def. 8); and compiled fair MAP task, in which we jointly solve the goal assignment and planning problems (Def. 9).

\subsection{Original Driverlog Task}

\textbf{Domain}

\begin{lstlisting}[language=pddl]
(define (domain driverlog)
	(:requirements :typing)
(:types
	location locatable - object
	driver truck package - locatable
)
(:predicates
	(in ?obj1 - package ?obj - truck)
	(path ?x - location ?y - location)
	(empty ?v - truck)
	(at ?obj - locatable ?loc - location)
	(link ?x - location ?y - location)
	(driving ?agent - driver ?v - truck)
)

(:action LOAD-TRUCK
	:parameters (?driver - driver ?truck - truck ?obj - package ?loc - location)
	:precondition (and
		(at ?truck ?loc)
		(at ?obj ?loc)
		(driving ?driver ?truck)
	)
	:effect (and
		(not (at ?obj ?loc))
		(in ?obj ?truck)
	)
)


(:action UNLOAD-TRUCK
	:parameters (?driver - driver ?truck - truck ?obj - package ?loc - location)
	:precondition (and
		(at ?truck ?loc)
		(in ?obj ?truck)
		(driving ?driver ?truck)
	)
	:effect (and
		(not (in ?obj ?truck))
		(at ?obj ?loc)
	)
)


(:action BOARD-TRUCK
	:parameters (?driver - driver ?truck - truck ?loc - location)
	:precondition (and
		(at ?truck ?loc)
		(at ?driver ?loc)
		(empty ?truck)
	)
	:effect (and
		(not (at ?driver ?loc))
		(driving ?driver ?truck)
		(not (empty ?truck))
	)
)


(:action DISEMBARK-TRUCK
	:parameters (?driver - driver ?truck - truck ?loc - location)
	:precondition (and
		(at ?truck ?loc)
		(driving ?driver ?truck)
	)
	:effect (and
		(not (driving ?driver ?truck))
		(at ?driver ?loc)
		(empty ?truck)
	)
)


(:action DRIVE-TRUCK
	:parameters (?driver - driver ?loc-from - location ?loc-to - location ?truck - truck)
	:precondition (and
		(at ?truck ?loc-from)
		(driving ?driver ?truck)
		(link ?loc-from ?loc-to)
	)
	:effect (and
		(not (at ?truck ?loc-from))
		(at ?truck ?loc-to)
	)
)


(:action WALK
	:parameters (?driver - driver ?loc-from - location ?loc-to - location)
	:precondition (and
		(at ?driver ?loc-from)
		(path ?loc-from ?loc-to)
	)
	:effect (and
		(not (at ?driver ?loc-from))
		(at ?driver ?loc-to)
	)
)

)
\end{lstlisting}

\textbf{Problem}

\begin{lstlisting}[language=pddl]
(define (problem DLOG-3-2-4) (:domain driverlog)
(:objects
	p1-2 - location
	driver2 - driver
	truck1 - truck
	truck2 - truck
	driver1 - driver
	s2 - location
	s1 - location
	s0 - location
	p0-1 - location
	package1 - package
	package2 - package
	package3 - package
	package4 - package
	driver3 - driver
)
(:init
	(at driver1 s1)
	(at driver2 s1)
	(at driver3 s0)
	(at truck1 s1)
	(empty truck1)
	(at truck2 s0)
	(empty truck2)
	(at package1 s2)
	(at package2 s2)
	(at package3 s0)
	(at package4 s1)
	(path s0 p0-1)
	(path p0-1 s0)
	(path s1 p0-1)
	(path p0-1 s1)
	(path s1 p1-2)
	(path p1-2 s1)
	(path s2 p1-2)
	(path p1-2 s2)
	(link s0 s2)
	(link s2 s0)
	(link s1 s0)
	(link s0 s1)
	(link s2 s1)
	(link s1 s2)
)
(:goal
	(and
		(at truck1 s1)
		(at truck2 s2)
		(at package1 s1)
		(at package2 s2)
		(at package3 s2)
		(at package4 s0)
	)
)
)
\end{lstlisting}

\subsection{Labeled Fair MAP task}

\textbf{Domain}

\begin{lstlisting}[language=pddl]
(define (domain driverlog)
(:requirements :typing)
(:types location - object
 locatable - object
 driver - locatable
 driver - object
 truck - locatable
 truck - object
 package - locatable
 package - object
)

(:predicates
(in ?obj1 - package ?obj - truck)
(path ?x - location ?y - location)
(empty ?v - truck)
(at ?obj - locatable ?loc - location)
(link ?x - location ?y - location)
(driving ?agent - driver ?v - truck)
(atAB ?obj - locatable ?loc - location ?a - locatable)
(attruck2-s2-done)
(atpackage1-s1-done)
(atpackage3-s2-done)
(atpackage4-s0-done)
)

(:functions
)

(:action load-truck
:parameters ( ?driver - driver ?truck - truck ?obj - package ?loc - location)
:precondition (and
(at ?truck ?loc)
(at ?obj ?loc)
(driving ?driver ?truck)
)
:effect (and
(not (at ?obj ?loc))
(in ?obj ?truck)
)
)

(:action unload-truck
:parameters ( ?driver - driver ?truck - truck ?obj - package ?loc - location)
:precondition (and
(at ?truck ?loc)
(in ?obj ?truck)
(driving ?driver ?truck)
)
:effect (and
(not (in ?obj ?truck))
(at ?obj ?loc)
(when
(and (= ?obj package1) (= ?loc s1) (not (atpackage1-s1-done)))
(and (atAB ?obj ?loc ?driver))
)
(when
(and (= ?obj package3) (= ?loc s2) (not (atpackage3-s2-done)))
(and (atAB ?obj ?loc ?driver))
)
(when
(and (= ?obj package4) (= ?loc s0) (not (atpackage4-s0-done)))
(and (atAB ?obj ?loc ?driver))
)
)
)

(:action board-truck
:parameters ( ?driver - driver ?truck - truck ?loc - location)
:precondition (and
(at ?truck ?loc)
(at ?driver ?loc)
(empty ?truck)
)
:effect (and
(not (at ?driver ?loc))
(driving ?driver ?truck)
(not (empty ?truck))
)
)

(:action disembark-truck
:parameters ( ?driver - driver ?truck - truck ?loc - location)
:precondition (and
(at ?truck ?loc)
(driving ?driver ?truck)
)
:effect (and
(not (driving ?driver ?truck))
(at ?driver ?loc)
(empty ?truck)
)
)

(:action drive-truck
:parameters ( ?driver - driver ?loc-from - location ?loc-to - location ?truck - truck)
:precondition (and
(at ?truck ?loc-from)
(driving ?driver ?truck)
(link ?loc-from ?loc-to)
)
:effect (and
(not (at ?truck ?loc-from))
(at ?truck ?loc-to)
(when
(and (= ?truck truck2) (= ?loc-to s2) (not (attruck2-s2-done)))
(and (atAB ?truck ?loc-to ?driver))
)
)
)

(:action walk
:parameters ( ?driver - driver ?loc-from - location ?loc-to - location)
:precondition (and
(at ?driver ?loc-from)
(path ?loc-from ?loc-to)
)
:effect (and
(not (at ?driver ?loc-from))
(at ?driver ?loc-to)
)
)

)
\end{lstlisting}

\textbf{Problem}

\begin{lstlisting}[language=pddl]
(define (problem dlog-3-2-4)
(:domain driverlog)
(:objects
p1-2 - location
driver2 - driver
truck1 - truck
truck2 - truck
driver1 - driver
s2 - location
s1 - location
s0 - location
p0-1 - location
package1 - package
package2 - package
package3 - package
package4 - package
driver3 - driver
)

(:init
(at package4 s1)
(at package3 s0)
(empty truck1)
(path p0-1 s1)
(at driver3 s0)
(link s2 s1)
(link s1 s2)
(path s2 p1-2)
(at driver2 s1)
(path p0-1 s0)
(link s0 s1)
(link s0 s2)
(link s2 s0)
(at truck2 s0)
(at driver1 s1)
(link s1 s0)
(at truck1 s1)
(path p1-2 s2)
(path s1 p0-1)
(at package1 s2)
(empty truck2)
(path s0 p0-1)
(path s1 p1-2)
(at package2 s2)
(path p1-2 s1)
)

(:goal (and
(atAB package1  s1 driver1)
(atAB package4  s0 driver2)
(atAB package3  s2 driver3)
(atAB truck2  s2 driver3)
(at truck1 s1)
(at truck2 s2)
(at package1 s1)
(at package2 s2)
(at package3 s2)
(at package4 s0)
))

)
\end{lstlisting}

\subsection{Compiled Fair MAP task}

\textbf{Domain}

\begin{lstlisting}[language=pddl]
(define (domain driverlog)
(:requirements :typing)
(:types location - object
 locatable - object
 driver - locatable
 driver - object
 driver - agent
 truck - locatable
 truck - object
 package - locatable
 package - object
 agent
 number
)

(:predicates
(in ?obj1 - package ?obj - truck)
(path ?x - location ?y - location)
(empty ?v - truck)
(at ?obj - locatable ?loc - location)
(link ?x - location ?y - location)
(driving ?agent - driver ?v - truck)
(next ?n1 - number ?n2 - number)
(end)
(n_goal_achieved ?a - driver ?n - number)
(attruck2-s2-done)
(atpackage1-s1-done)
(atpackage3-s2-done)
(atpackage4-s0-done)
)

(:functions
(min-associated-cost ?n - number)
)

(:action load-truck
:parameters ( ?driver - driver ?truck - truck ?obj - package ?loc - location)
:precondition (and
(at ?truck ?loc)
(at ?obj ?loc)
(driving ?driver ?truck)
)
:effect (and
(not (at ?obj ?loc))
(in ?obj ?truck)
(increase (total-cost) 1)
)
)

(:action unload-truck
:parameters ( ?driver - driver ?truck - truck ?obj - package ?loc - location ?n1 - number ?n2 - number)
:precondition (and
(at ?truck ?loc)
(in ?obj ?truck)
(driving ?driver ?truck)
(n_goal_achieved ?driver ?n1)
(next ?n1 ?n2)
)
:effect (and
(not (in ?obj ?truck))
(at ?obj ?loc)
(when
(and (= ?obj package1) (= ?loc s1) (not (atpackage1-s1-done)))
(and (not (n_goal_achieved ?driver ?n1)) (n_goal_achieved ?driver ?n2) (atpackage1-s1-done))
)
(when
(and (= ?obj package3) (= ?loc s2) (not (atpackage3-s2-done)))
(and (not (n_goal_achieved ?driver ?n1)) (n_goal_achieved ?driver ?n2) (atpackage3-s2-done))
)
(when
(and (= ?obj package4) (= ?loc s0) (not (atpackage4-s0-done)))
(and (not (n_goal_achieved ?driver ?n1)) (n_goal_achieved ?driver ?n2) (atpackage4-s0-done))
)
(increase (total-cost) 1)
)
)

(:action board-truck
:parameters ( ?driver - driver ?truck - truck ?loc - location)
:precondition (and
(at ?truck ?loc)
(at ?driver ?loc)
(empty ?truck)
)
:effect (and
(not (at ?driver ?loc))
(driving ?driver ?truck)
(not (empty ?truck))
(increase (total-cost) 1)
)
)

(:action disembark-truck
:parameters ( ?driver - driver ?truck - truck ?loc - location)
:precondition (and
(at ?truck ?loc)
(driving ?driver ?truck)
)
:effect (and
(not (driving ?driver ?truck))
(at ?driver ?loc)
(empty ?truck)
(increase (total-cost) 1)
)
)

(:action drive-truck
:parameters ( ?driver - driver ?loc-from - location ?loc-to - location ?truck - truck ?n1 - number ?n2 - number)
:precondition (and
(at ?truck ?loc-from)
(driving ?driver ?truck)
(link ?loc-from ?loc-to)
(n_goal_achieved ?driver ?n1)
(next ?n1 ?n2)
)
:effect (and
(not (at ?truck ?loc-from))
(at ?truck ?loc-to)
(when
(and (= ?truck truck2) (= ?loc-to s2) (not (attruck2-s2-done)))
(and (not (n_goal_achieved ?driver ?n1)) (n_goal_achieved ?driver ?n2) (attruck2-s2-done))
)
(increase (total-cost) 1)
)
)

(:action walk
:parameters ( ?driver - driver ?loc-from - location ?loc-to - location)
:precondition (and
(at ?driver ?loc-from)
(path ?loc-from ?loc-to)
)
:effect (and
(not (at ?driver ?loc-from))
(at ?driver ?loc-to)
(increase (total-cost) 1)
)
)

(:action __give_min_reward_0-0-4
:parameters ( ?a0 - agent ?a1 - agent ?a2 - agent)
:precondition (and
(at truck2 s2)
(attruck2-s2-done)
(at package1 s1)
(atpackage1-s1-done)
(at package3 s2)
(atpackage3-s2-done)
(at package4 s0)
(atpackage4-s0-done)
(at package2 s2)
(at truck1 s1)
(not (= ?a0 ?a1))
(not (= ?a0 ?a2))
(not (= ?a1 ?a2))
(n_goal_achieved ?a0 n0)
(n_goal_achieved ?a1 n0)
(n_goal_achieved ?a2 n4)
)
:effect (and
(end)
(increase (total-cost) (min-associated-cost n0))
)
)

(:action __give_min_reward_0-1-3
:parameters ( ?a0 - agent ?a1 - agent ?a2 - agent)
:precondition (and
(at truck2 s2)
(attruck2-s2-done)
(at package1 s1)
(atpackage1-s1-done)
(at package3 s2)
(atpackage3-s2-done)
(at package4 s0)
(atpackage4-s0-done)
(at package2 s2)
(at truck1 s1)
(not (= ?a0 ?a1))
(not (= ?a0 ?a2))
(not (= ?a1 ?a2))
(n_goal_achieved ?a0 n0)
(n_goal_achieved ?a1 n1)
(n_goal_achieved ?a2 n3)
)
:effect (and
(end)
(increase (total-cost) (min-associated-cost n0))
)
)

(:action __give_min_reward_0-2-2
:parameters ( ?a0 - agent ?a1 - agent ?a2 - agent)
:precondition (and
(at truck2 s2)
(attruck2-s2-done)
(at package1 s1)
(atpackage1-s1-done)
(at package3 s2)
(atpackage3-s2-done)
(at package4 s0)
(atpackage4-s0-done)
(at package2 s2)
(at truck1 s1)
(not (= ?a0 ?a1))
(not (= ?a0 ?a2))
(not (= ?a1 ?a2))
(n_goal_achieved ?a0 n0)
(n_goal_achieved ?a1 n2)
(n_goal_achieved ?a2 n2)
)
:effect (and
(end)
(increase (total-cost) (min-associated-cost n0))
)
)

(:action __give_min_reward_1-1-2
:parameters ( ?a0 - agent ?a1 - agent ?a2 - agent)
:precondition (and
(at truck2 s2)
(attruck2-s2-done)
(at package1 s1)
(atpackage1-s1-done)
(at package3 s2)
(atpackage3-s2-done)
(at package4 s0)
(atpackage4-s0-done)
(at package2 s2)
(at truck1 s1)
(not (= ?a0 ?a1))
(not (= ?a0 ?a2))
(not (= ?a1 ?a2))
(n_goal_achieved ?a0 n1)
(n_goal_achieved ?a1 n1)
(n_goal_achieved ?a2 n2)
)
:effect (and
(end)
(increase (total-cost) (min-associated-cost n1))
)
)
)
\end{lstlisting}

\textbf{Problem}

\begin{lstlisting}[language=pddl]
(define (problem dlog-3-2-4)
(:domain driverlog)
(:objects
p1-2 - location
driver2 - driver
truck1 - truck
truck2 - truck
driver1 - driver
s2 - location
s1 - location
s0 - location
p0-1 - location
package1 - package
package2 - package
package3 - package
package4 - package
driver3 - driver
n0 - number
n1 - number
n2 - number
n3 - number
n4 - number
n5 - number
)

(:init
(at truck1 s1)
(link s1 s0)
(path s0 p0-1)
(empty truck2)
(at package4 s1)
(at package2 s2)
(at package3 s0)
(link s0 s1)
(link s1 s2)
(at driver2 s1)
(at driver3 s0)
(path p0-1 s1)
(path p1-2 s2)
(at package1 s2)
(at driver1 s1)
(link s2 s1)
(path s1 p0-1)
(path s2 p1-2)
(path p1-2 s1)
(link s0 s2)
(at truck2 s0)
(path s1 p1-2)
(link s2 s0)
(empty truck1)
(path p0-1 s0)
(n_goal_achieved driver1 n0)
(n_goal_achieved driver3 n0)
(n_goal_achieved driver2 n0)
(next n0 n1)
(next n1 n2)
(next n2 n3)
(next n3 n4)
(next n4 n5)
(= (min-associated-cost n0) 6000)
(= (min-associated-cost n1) 5000)
(= (min-associated-cost n2) 4000)
(= (min-associated-cost n3) 3000)
(= (min-associated-cost n4) 2000)
(= (min-associated-cost n5) 1000)
)

(:goal (and
(end)
))
(:metric minimize (total-cost))
)
\end{lstlisting}

\newpage

\section*{Appendix B }\label{AppendixB}
In this section we present the results of the algorithms on each domain.

\begin{table*}[h]
\tiny
    \centering
    \begin{tabular}{|c|c|c|c|r|r|r|r|r|r|r|}
    \hline
\textbf{Problems} & $\mathbf{|N|}$ & $\mathbf{|G|}$ &\textbf{Approach} & \textbf{Plan Cost} & \textbf{G-Maximin} & \textbf{G-Propeq} & \textbf{W-Maximin}& \textbf{W-Propeq} & \textbf{Total Time} & \textbf{Coverage}\\ \hline
\multirow{9}{*}{blocksworld} & \multirow{9}{*}{$4.0\pm0.0$} & \multirow{9}{*}{$11.35\pm2.56$} & \textsc{lama} & $19.44$ & $11.0$ & $5.54$ & $10.82$ & $7.19$ & $20.0$ & $20$ \\
& & & \textsc{madagascar} &  $0$ & $0$ & $0$ & $0$ & $0$ & $0$ & $0$ \\
& & & \textsc{contract-net} &  $17.98$ & $18.83$ & $14.5$ & $18.1$ & $13.29$ & $20.0$ & $20$ \\
& & & \textsc{milp-g-maximin} &  $18.15$ & $18.33$ & $12.0$ & $17.28$ & $12.58$ & $20.0$ & $20$ \\
& & & \textsc{milp-g-propeq} &  $18.18$ & $19.67$ & $16.0$ & $17.82$ & $12.95$ & $20.0$ & $20$ \\
& & & \textsc{milp-w-maximin} &  $18.67$ & $18.42$ & $12.75$ & $17.8$ & $13.67$ & $20.0$ & $20$ \\
& & & \textsc{milp-w-propeq} &  $17.9$ & $18.5$ & $13.83$ & $17.62$ & $14.29$ & $20.0$ & $20$ \\
& & & \textsc{fpc-g-maximin} &  $15.84$ & $18.67$ & $16.33$ & $10.06$ & $11.7$ & $20.0$ & $20$ \\
& & & \textsc{fpc-g-propeq} &  $16.14$ & $18.67$ & $18.0$ & $10.33$ & $13.68$ & $20.0$ & $20$ \\
\hline
\multirow{9}{*}{depot} & \multirow{9}{*}{$6.85\pm1.81$} & \multirow{9}{*}{$8.15\pm3.72$} & \textsc{lama} & $16.6$ & $17$ & $13.17$ & $3.39$ & $11.89$ & $20.08$ & $17$ \\
& & & \textsc{madagascar} &  $12.75$ & $20$ & $15.67$ & $19.38$ & $17.12$ & $35.33$ & $20$ \\
& & & \textsc{contract-net} &  $11.7$ & $16$ & $13.75$ & $4.01$ & $8.81$ & $24.07$ & $16$ \\
& & & \textsc{milp-g-maximin} &  $12.6$ & $16$ & $13.27$ & $3.4$ & $8.47$ & $24.07$ & $16$ \\
& & & \textsc{milp-g-propeq} &  $12.03$ & $16$ & $15.75$ & $3.85$ & $9.17$ & $24.07$ & $16$ \\
& & & \textsc{milp-w-maximin} &  $12.6$ & $16$ & $13.27$ & $3.4$ & $8.47$ & $24.07$ & $16$ \\
& & & \textsc{milp-w-propeq} &  $11.73$ & $16$ & $15.08$ & $5.39$ & $9.26$ & $24.07$ & $16$ \\
& & & \textsc{fpc-g-maximin} &  $7.45$ & $9$ & $7.6$ & $2.57$ & $7.07$ & $11.38$ & $9$ \\
& & & \textsc{fpc-g-propeq} &  $7.32$ & $9$ & $9$ & $0.8$ & $5.95$ & $11.21$ & $9$ \\
\hline
\multirow{9}{*}{driverlog} & \multirow{9}{*}{$3.3\pm1.59$} & \multirow{9}{*}{$10.75\pm7.61$} & \textsc{lama} & $20.0$ & $4.0$ & $4.43$ & $1.39$ & $5.13$ & $34.07$ & $20$ \\
& & & \textsc{madagascar} &  $12.48$ & $11.75$ & $9.78$ & $17.6$ & $10.89$ & $37.03$ & $20$ \\
& & & \textsc{contract-net} &  $14.82$ & $12.75$ & $11.08$ & $10.87$ & $7.83$ & $32.14$ & $19$ \\
& & & \textsc{milp-g-maximin} &  $14.97$ & $17.83$ & $16.28$ & $13.41$ & $11.18$ & $32.39$ & $20$ \\
& & & \textsc{milp-g-propeq} &  $14.33$ & $17.67$ & $15.23$ & $12.67$ & $9.63$ & $31.42$ & $19$ \\
& & & \textsc{milp-w-maximin} &  $14.17$ & $16.75$ & $14.9$ & $13.29$ & $9.56$ & $31.22$ & $19$ \\
& & & \textsc{milp-w-propeq} &  $14.81$ & $17.58$ & $14.45$ & $13.35$ & $9.67$ & $32.28$ & $20$ \\
& & & \textsc{fpc-g-maximin} &  $14.47$ & $18$ & $16.0$ & $4.95$ & $15.95$ & $30.06$ & $18$ \\
& & & \textsc{fpc-g-propeq} &  $14.17$ & $18$ & $18$ & $5.15$ & $15.83$ & $33.87$ & $18$ \\
\hline
\multirow{9}{*}{elevators08} & \multirow{9}{*}{$4.0\pm0.0$} & \multirow{9}{*}{$12.75\pm6.38$} & \textsc{lama} & $14.53$ & $3.5$ & $10.27$ & $1.1$ & $9.96$ & $26.17$ & $20$ \\
& & & \textsc{madagascar} &  $15.75$ & $14.0$ & $12.25$ & $17$ & $4.89$ & $48.76$ & $17$ \\
& & & \textsc{contract-net} &  $12.46$ & $9.33$ & $11.49$ & $5.5$ & $12.49$ & $24.05$ & $20$ \\
& & & \textsc{milp-g-maximin} &  $12.78$ & $9.17$ & $12.43$ & $5.86$ & $12.95$ & $22.35$ & $20$ \\
& & & \textsc{milp-g-propeq} &  $12.46$ & $10.0$ & $12.75$ & $6.94$ & $13.36$ & $22.28$ & $20$ \\
& & & \textsc{milp-w-maximin} &  $11.79$ & $10.83$ & $12.81$ & $7.92$ & $15.57$ & $22.28$ & $20$ \\
& & & \textsc{milp-w-propeq} &  $11.22$ & $8.33$ & $11.63$ & $5.83$ & $13.78$ & $22.28$ & $20$ \\
& & & \textsc{fpc-g-maximin} &  $8.64$ & $15.83$ & $15.49$ & $4.8$ & $12.5$ & $19.0$ & $19$ \\
& & & \textsc{fpc-g-propeq} &  $8.41$ & $15.83$ & $16.74$ & $4.87$ & $11.9$ & $19.0$ & $19$ \\
\hline
\multirow{9}{*}{logistics00} & \multirow{9}{*}{$5.05\pm1.47$} & \multirow{9}{*}{$10.3\pm3.28$} & \textsc{lama} & $19.87$ & $19.0$ & $18.32$ & $10.77$ & $14.33$ & $28.9$ & $20$ \\
& & & \textsc{madagascar} &  $12.62$ & $20$ & $19.07$ & $19.71$ & $11.27$ & $20$ & $20$ \\
& & & \textsc{contract-net} &  $18.95$ & $20$ & $20$ & $12.24$ & $14.95$ & $31.24$ & $20$ \\
& & & \textsc{milp-g-maximin} &  $19.3$ & $20$ & $19.55$ & $11.83$ & $15.51$ & $30.54$ & $20$ \\
& & & \textsc{milp-g-propeq} &  $18.37$ & $20$ & $19.55$ & $12.96$ & $15.12$ & $29.82$ & $20$ \\
& & & \textsc{milp-w-maximin} &  $19.01$ & $20$ & $19.3$ & $12.4$ & $14.46$ & $30.18$ & $20$ \\
& & & \textsc{milp-w-propeq} &  $18.05$ & $20$ & $19.17$ & $13.72$ & $14.76$ & $29.92$ & $20$ \\
& & & \textsc{fpc-g-maximin} &  $13.58$ & $14$ & $14$ & $4.75$ & $13.38$ & $21.43$ & $14$ \\
& & & \textsc{fpc-g-propeq} &  $13.51$ & $14$ & $14$ & $4.75$ & $13.28$ & $21.29$ & $14$ \\
\hline
\multirow{9}{*}{rovers} & \multirow{9}{*}{$6.6\pm2.44$} & \multirow{9}{*}{$13.85\pm4.45$} & \textsc{lama} & $19.86$ & $2.0$ & $5.17$ & $0.86$ & $6.51$ & $20.0$ & $20$ \\
& & & \textsc{madagascar} &  $11.33$ & $8.83$ & $8.8$ & $18.64$ & $9.75$ & $55.59$ & $20$ \\
& & & \textsc{contract-net} &  $18.1$ & $16.5$ & $12.42$ & $10.5$ & $12.17$ & $20.0$ & $20$ \\
& & & \textsc{milp-g-maximin} &  $18.37$ & $19.5$ & $12.52$ & $12.16$ & $11.53$ & $20.0$ & $20$ \\
& & & \textsc{milp-g-propeq} &  $18.21$ & $18.5$ & $18.92$ & $12.7$ & $14.98$ & $20.0$ & $20$ \\
& & & \textsc{milp-w-maximin} &  $17.01$ & $18.0$ & $18.58$ & $15.44$ & $15.78$ & $20.0$ & $20$ \\
& & & \textsc{milp-w-propeq} &  $17.31$ & $16.5$ & $17.25$ & $14.97$ & $15.36$ & $20.0$ & $20$ \\
& & & \textsc{fpc-g-maximin} &  $11.36$ & $10.17$ & $5.87$ & $3.81$ & $7.45$ & $12.0$ & $12$ \\
& & & \textsc{fpc-g-propeq} &  $11.06$ & $9.17$ & $8.52$ & $4.35$ & $8.46$ & $12.0$ & $12$ \\
\hline
\multirow{9}{*}{satellites} & \multirow{9}{*}{$5.5\pm1.73$} & \multirow{9}{*}{$24.4\pm18.09$} & \textsc{lama} & $19.94$ & $1.0$ & $3.0$ & $0.03$ & $2.96$ & $20.0$ & $20$ \\
& & & \textsc{madagascar} &  $10.24$ & $14.46$ & $10.91$ & $19.4$ & $11.02$ & $37.43$ & $20$ \\
& & & \textsc{contract-net} &  $14.23$ & $16.0$ & $15.67$ & $12.94$ & $10.81$ & $20.0$ & $20$ \\
& & & \textsc{milp-g-maximin} &  $12.84$ & $10.69$ & $8.64$ & $13.19$ & $6.2$ & $20.0$ & $20$ \\
& & & \textsc{milp-g-propeq} &  $13.41$ & $13.49$ & $14.16$ & $13.41$ & $8.37$ & $20.0$ & $20$ \\
& & & \textsc{milp-w-maximin} &  $13.46$ & $14.14$ & $11.52$ & $13.39$ & $8.65$ & $20.0$ & $20$ \\
& & & \textsc{milp-w-propeq} &  $13.41$ & $13.63$ & $14.02$ & $13.44$ & $9.1$ & $20.0$ & $20$ \\
& & & \textsc{fpc-g-maximin} &  $10.41$ & $10.79$ & $7.4$ & $5.57$ & $8.78$ & $13.0$ & $13$ \\
& & & \textsc{fpc-g-propeq} &  $9.99$ & $11.04$ & $11.16$ & $6.07$ & $10.11$ & $13.0$ & $13$ \\
\hline
\multirow{9}{*}{sokoban} & \multirow{9}{*}{$2.6\pm0.6$} & \multirow{9}{*}{$3.2\pm1.01$} & \textsc{lama} & $16.94$ & $9.0$ & $11.25$ & $8.0$ & $6.34$ & $38.59$ & $18$ \\
& & & \textsc{madagascar} &  $9.15$ & $8.0$ & $10.0$ & $14.65$ & $12.08$ & $48.08$ & $15$ \\
& & & \textsc{contract-net} &  $14.8$ & $10.0$ & $11.0$ & $8.91$ & $6.92$ & $59.34$ & $16$ \\
& & & \textsc{milp-g-maximin} &  $15.02$ & $14.0$ & $14.33$ & $7.99$ & $7.68$ & $58.88$ & $16$ \\
& & & \textsc{milp-g-propeq} &  $14.1$ & $13.0$ & $13.33$ & $7.58$ & $7.36$ & $58.0$ & $15$ \\
& & & \textsc{milp-w-maximin} &  $14.71$ & $14.5$ & $14.0$ & $9.94$ & $8.54$ & $54.12$ & $17$ \\
& & & \textsc{milp-w-propeq} &  $15.03$ & $15.0$ & $15.0$ & $9.48$ & $8.81$ & $56.35$ & $18$ \\
& & & \textsc{fpc-g-maximin} &  $16.38$ & $14.0$ & $14.33$ & $6.17$ & $6.38$ & $41.05$ & $17$ \\
& & & \textsc{fpc-g-propeq} &  $16.33$ & $14.0$ & $14.33$ & $6.24$ & $6.5$ & $33.69$ & $18$ \\
\hline
\multirow{9}{*}{zenotravel} & \multirow{9}{*}{$3.8\pm1.4$} & \multirow{9}{*}{$13.9\pm9.68$} & \textsc{lama} & $20.0$ & $0.5$ & $1.62$ & $0.26$ & $2.64$ & $36.25$ & $20$ \\
& & & \textsc{madagascar} &  $7.64$ & $7.83$ & $3.58$ & $12.12$ & $4.5$ & $27.67$ & $14$ \\
& & & \textsc{contract-net} &  $14.5$ & $16.58$ & $12.7$ & $15.14$ & $11.93$ & $28.13$ & $20$ \\
& & & \textsc{milp-g-maximin} &  $14.85$ & $17.0$ & $13.83$ & $14.55$ & $12.93$ & $28.12$ & $20$ \\
& & & \textsc{milp-g-propeq} &  $14.89$ & $16.67$ & $13.65$ & $15.1$ & $14.22$ & $28.13$ & $20$ \\
& & & \textsc{milp-w-maximin} &  $14.08$ & $13.58$ & $10.83$ & $15.21$ & $11.18$ & $28.98$ & $20$ \\
& & & \textsc{milp-w-propeq} &  $14.91$ & $15.83$ & $11.98$ & $14.94$ & $11.31$ & $29.05$ & $20$ \\
& & & \textsc{fpc-g-maximin} &  $15.17$ & $18.0$ & $15.37$ & $9.06$ & $11.16$ & $22.12$ & $19$ \\
& & & \textsc{fpc-g-propeq} &  $15.2$ & $17.67$ & $17.45$ & $9.15$ & $12.02$ & $33.68$ & $19$ \\
\hline
    \end{tabular}
    \caption{Score obtained by each approach on each of the metrics for all the problems.}
    \label{tab:fairness_table3}
\end{table*}

\begin{table*}[h]
\tiny
    \centering
    \begin{tabular}{|c|c|c|c|c|c|c|c|c|c|c|}
    \hline
\textbf{Domain} & $\mathbf{|N|}$ & $\mathbf{|G|}$ &\textbf{Approach} & \textbf{Plan Cost} & \textbf{Maximin} & \textbf{Propeq} & \textbf{Total Time} & \textbf{Coverage}\\ \hline
\multirow{9}{*}{blocksworld} & \multirow{9}{*}{$-\pm-$} & \multirow{9}{*}{$-\pm-$} & \textsc{lama} & $0$ & $0$ & $0$ & $0$ & $0$ & $0$ & $0$ \\
& & & \textsc{madagascar} &  $0$ & $0$ & $0$ & $0$ & $0$ & $0$ & $0$ \\
& & & \textsc{contract-net} &  $0$ & $0$ & $0$ & $0$ & $0$ & $0$ & $0$ \\
& & & \textsc{milp-g-maximin} &  $0$ & $0$ & $0$ & $0$ & $0$ & $0$ & $0$ \\
& & & \textsc{milp-g-propeq} &  $0$ & $0$ & $0$ & $0$ & $0$ & $0$ & $0$ \\
& & & \textsc{milp-w-maximin} &  $0$ & $0$ & $0$ & $0$ & $0$ & $0$ & $0$ \\
& & & \textsc{milp-w-propeq} &  $0$ & $0$ & $0$ & $0$ & $0$ & $0$ & $0$ \\
& & & \textsc{fpc-g-maximin} &  $0$ & $0$ & $0$ & $0$ & $0$ & $0$ & $0$ \\
& & & \textsc{fpc-g-propeq} &  $0$ & $0$ & $0$ & $0$ & $0$ & $0$ & $0$ \\
\hline
\multirow{9}{*}{depot} & \multirow{9}{*}{$5.33\pm1.0$} & \multirow{9}{*}{$6.44\pm3.28$} & \textsc{lama} & $9.0$ & $9$ & $7.33$ & $0.91$ & $5.1$ & $12.08$ & $9$ \\
& & & \textsc{madagascar} &  $5.46$ & $9$ & $7.33$ & $9$ & $7.93$ & $9$ & $9$ \\
& & & \textsc{contract-net} &  $7.09$ & $9$ & $7.42$ & $0.84$ & $4.28$ & $17.07$ & $9$ \\
& & & \textsc{milp-g-maximin} &  $7.61$ & $9$ & $7.68$ & $0.78$ & $4.5$ & $17.07$ & $9$ \\
& & & \textsc{milp-g-propeq} &  $6.93$ & $9$ & $8.75$ & $1.22$ & $4.47$ & $17.07$ & $9$ \\
& & & \textsc{milp-w-maximin} &  $7.61$ & $9$ & $7.68$ & $0.78$ & $4.5$ & $17.07$ & $9$ \\
& & & \textsc{milp-w-propeq} &  $7.09$ & $9$ & $8.42$ & $1.72$ & $4.49$ & $17.07$ & $9$ \\
& & & \textsc{fpc-g-maximin} &  $7.45$ & $9$ & $7.6$ & $2.57$ & $7.07$ & $11.38$ & $9$ \\
& & & \textsc{fpc-g-propeq} &  $7.32$ & $9$ & $9$ & $0.8$ & $5.95$ & $11.21$ & $9$ \\
\hline
\multirow{9}{*}{driverlog} & \multirow{9}{*}{$2.94\pm1.11$} & \multirow{9}{*}{$8.78\pm4.83$} & \textsc{lama} & $18.0$ & $3.0$ & $4.02$ & $1.39$ & $4.58$ & $32.07$ & $18$ \\
& & & \textsc{madagascar} &  $11.35$ & $9.75$ & $8.16$ & $16.28$ & $8.89$ & $23.84$ & $18$ \\
& & & \textsc{contract-net} &  $14.19$ & $12.0$ & $10.41$ & $10.16$ & $7.65$ & $31.14$ & $18$ \\
& & & \textsc{milp-g-maximin} &  $13.96$ & $16.83$ & $14.92$ & $12.08$ & $10.25$ & $30.39$ & $18$ \\
& & & \textsc{milp-g-propeq} &  $13.91$ & $16.67$ & $14.92$ & $11.67$ & $9.01$ & $30.42$ & $18$ \\
& & & \textsc{milp-w-maximin} &  $13.75$ & $16.5$ & $14.67$ & $12.51$ & $9.46$ & $30.22$ & $18$ \\
& & & \textsc{milp-w-propeq} &  $13.9$ & $15.83$ & $13.28$ & $11.82$ & $8.97$ & $30.28$ & $18$ \\
& & & \textsc{fpc-g-maximin} &  $14.47$ & $18$ & $16.0$ & $4.95$ & $15.95$ & $30.06$ & $18$ \\
& & & \textsc{fpc-g-propeq} &  $14.17$ & $18$ & $18$ & $5.15$ & $15.83$ & $33.87$ & $18$ \\
\hline
\multirow{9}{*}{elevators08} & \multirow{9}{*}{$4.0\pm0.0$} & \multirow{9}{*}{$10.76\pm4.47$} & \textsc{lama} & $11.77$ & $2.5$ & $7.61$ & $0.24$ & $7.91$ & $23.17$ & $17$ \\
& & & \textsc{madagascar} &  $15.75$ & $14.0$ & $12.25$ & $17$ & $4.89$ & $48.76$ & $17$ \\
& & & \textsc{contract-net} &  $9.5$ & $6.33$ & $8.49$ & $3.24$ & $9.76$ & $21.05$ & $17$ \\
& & & \textsc{milp-g-maximin} &  $9.92$ & $6.17$ & $9.43$ & $3.32$ & $9.98$ & $19.35$ & $17$ \\
& & & \textsc{milp-g-propeq} &  $9.72$ & $7.0$ & $9.75$ & $3.94$ & $10.42$ & $19.28$ & $17$ \\
& & & \textsc{milp-w-maximin} &  $9.1$ & $7.83$ & $9.91$ & $5.72$ & $13.07$ & $19.28$ & $17$ \\
& & & \textsc{milp-w-propeq} &  $8.41$ & $6.83$ & $8.97$ & $4.02$ & $10.42$ & $19.28$ & $17$ \\
& & & \textsc{fpc-g-maximin} &  $10.17$ & $14.83$ & $13.58$ & $3.83$ & $10.86$ & $17.0$ & $17$ \\
& & & \textsc{fpc-g-propeq} &  $9.95$ & $14.83$ & $14.83$ & $3.91$ & $10.26$ & $17.0$ & $17$ \\
\hline
\multirow{9}{*}{logistics00} & \multirow{9}{*}{$4.21\pm0.8$} & \multirow{9}{*}{$8.71\pm2.52$} & \textsc{lama} & $13.98$ & $14$ & $14$ & $7.9$ & $9.72$ & $22.9$ & $14$ \\
& & & \textsc{madagascar} &  $8.87$ & $14$ & $14$ & $14$ & $7.13$ & $14$ & $14$ \\
& & & \textsc{contract-net} &  $13.36$ & $14$ & $14$ & $8.51$ & $9.49$ & $25.24$ & $14$ \\
& & & \textsc{milp-g-maximin} &  $13.55$ & $14$ & $14$ & $8.65$ & $10.85$ & $24.54$ & $14$ \\
& & & \textsc{milp-g-propeq} &  $12.85$ & $14$ & $14$ & $9.32$ & $10.04$ & $23.82$ & $14$ \\
& & & \textsc{milp-w-maximin} &  $13.22$ & $14$ & $14$ & $9.03$ & $9.87$ & $24.18$ & $14$ \\
& & & \textsc{milp-w-propeq} &  $12.58$ & $14$ & $14$ & $9.61$ & $9.6$ & $23.92$ & $14$ \\
& & & \textsc{fpc-g-maximin} &  $13.58$ & $14$ & $14$ & $4.75$ & $13.38$ & $21.43$ & $14$ \\
& & & \textsc{fpc-g-propeq} &  $13.51$ & $14$ & $14$ & $4.75$ & $13.28$ & $21.29$ & $14$ \\
\hline
\multirow{9}{*}{rovers} & \multirow{9}{*}{$4.83\pm1.03$} & \multirow{9}{*}{$11.5\pm3.18$} & \textsc{lama} & $11.86$ & $2.0$ & $3.55$ & $0.86$ & $3.7$ & $12.0$ & $12$ \\
& & & \textsc{madagascar} &  $6.98$ & $6.33$ & $5.85$ & $11.58$ & $5.67$ & $18.29$ & $12$ \\
& & & \textsc{contract-net} &  $11.14$ & $10.5$ & $7.75$ & $6.4$ & $7.12$ & $12.0$ & $12$ \\
& & & \textsc{milp-g-maximin} &  $11.14$ & $12$ & $8.48$ & $7.49$ & $7.08$ & $12.0$ & $12$ \\
& & & \textsc{milp-g-propeq} &  $11.17$ & $11.5$ & $11.67$ & $7.6$ & $8.75$ & $12.0$ & $12$ \\
& & & \textsc{milp-w-maximin} &  $10.38$ & $10.0$ & $10.58$ & $8.85$ & $8.82$ & $12.0$ & $12$ \\
& & & \textsc{milp-w-propeq} &  $10.57$ & $12$ & $12$ & $8.76$ & $9.09$ & $12.0$ & $12$ \\
& & & \textsc{fpc-g-maximin} &  $11.36$ & $10.17$ & $5.87$ & $3.81$ & $7.45$ & $12.0$ & $12$ \\
& & & \textsc{fpc-g-propeq} &  $11.06$ & $9.17$ & $8.52$ & $4.35$ & $8.46$ & $12.0$ & $12$ \\
\hline
\multirow{9}{*}{satellites} & \multirow{9}{*}{$5.0\pm1.31$} & \multirow{9}{*}{$22.27\pm19.3$} & \textsc{lama} & $12.94$ & $1.0$ & $2.17$ & $0.03$ & $2.01$ & $13.0$ & $13$ \\
& & & \textsc{madagascar} &  $6.74$ & $9.4$ & $6.93$ & $13$ & $6.34$ & $13$ & $13$ \\
& & & \textsc{contract-net} &  $9.75$ & $12.0$ & $11.24$ & $8.39$ & $7.03$ & $13.0$ & $13$ \\
& & & \textsc{milp-g-maximin} &  $9.46$ & $10.29$ & $7.6$ & $8.1$ & $4.39$ & $13.0$ & $13$ \\
& & & \textsc{milp-g-propeq} &  $9.54$ & $11.29$ & $11.21$ & $8.29$ & $4.85$ & $13.0$ & $13$ \\
& & & \textsc{milp-w-maximin} &  $9.46$ & $11.14$ & $8.67$ & $8.51$ & $5.58$ & $13.0$ & $13$ \\
& & & \textsc{milp-w-propeq} &  $9.46$ & $11.43$ & $11.23$ & $8.36$ & $5.71$ & $13.0$ & $13$ \\
& & & \textsc{fpc-g-maximin} &  $10.41$ & $10.79$ & $7.4$ & $5.57$ & $8.78$ & $13.0$ & $13$ \\
& & & \textsc{fpc-g-propeq} &  $9.99$ & $11.04$ & $11.16$ & $6.07$ & $10.11$ & $13.0$ & $13$ \\
\hline
\multirow{9}{*}{sokoban} & \multirow{9}{*}{$2.38\pm0.51$} & \multirow{9}{*}{$3.23\pm0.73$} & \textsc{lama} & $12.98$ & $6.0$ & $7.58$ & $4.69$ & $4.16$ & $31.07$ & $13$ \\
& & & \textsc{madagascar} &  $7.94$ & $8.0$ & $9.17$ & $12.65$ & $10.08$ & $40.9$ & $13$ \\
& & & \textsc{contract-net} &  $12.02$ & $7.0$ & $8.33$ & $6.6$ & $5.81$ & $53.77$ & $13$ \\
& & & \textsc{milp-g-maximin} &  $12.27$ & $11.0$ & $11.67$ & $5.58$ & $5.72$ & $53.1$ & $13$ \\
& & & \textsc{milp-g-propeq} &  $12.27$ & $11.0$ & $11.67$ & $5.58$ & $5.72$ & $53.19$ & $13$ \\
& & & \textsc{milp-w-maximin} &  $11.01$ & $10.5$ & $10.33$ & $7.31$ & $6.34$ & $47.29$ & $13$ \\
& & & \textsc{milp-w-propeq} &  $11.39$ & $11.0$ & $11.33$ & $6.85$ & $6.66$ & $49.52$ & $13$ \\
& & & \textsc{fpc-g-maximin} &  $12.65$ & $11.0$ & $11.0$ & $3.43$ & $4.76$ & $35.16$ & $13$ \\
& & & \textsc{fpc-g-propeq} &  $12.65$ & $11.0$ & $11.0$ & $3.58$ & $4.72$ & $27.82$ & $13$ \\
\hline
\multirow{9}{*}{zenotravel} & \multirow{9}{*}{$3.21\pm1.25$} & \multirow{9}{*}{$8.43\pm4.8$} & \textsc{lama} & $14.0$ & $0.5$ & $1.16$ & $0.21$ & $1.73$ & $30.25$ & $14$ \\
& & & \textsc{madagascar} &  $7.64$ & $7.83$ & $3.58$ & $12.12$ & $4.5$ & $27.67$ & $14$ \\
& & & \textsc{contract-net} &  $10.5$ & $12.67$ & $10.5$ & $9.88$ & $8.6$ & $22.13$ & $14$ \\
& & & \textsc{milp-g-maximin} &  $10.76$ & $11.83$ & $10.5$ & $9.05$ & $8.3$ & $22.12$ & $14$ \\
& & & \textsc{milp-g-propeq} &  $10.77$ & $11.67$ & $11.08$ & $9.66$ & $9.89$ & $22.13$ & $14$ \\
& & & \textsc{milp-w-maximin} &  $10.16$ & $10.33$ & $8.33$ & $9.96$ & $7.96$ & $22.98$ & $14$ \\
& & & \textsc{milp-w-propeq} &  $10.83$ & $12.5$ & $10.58$ & $9.42$ & $7.29$ & $23.05$ & $14$ \\
& & & \textsc{fpc-g-maximin} &  $11.57$ & $13.5$ & $12.33$ & $6.15$ & $9.2$ & $17.12$ & $14$ \\
& & & \textsc{fpc-g-propeq} &  $11.77$ & $13.17$ & $13.25$ & $6.28$ & $10.33$ & $28.68$ & $14$ \\
\hline
    \end{tabular}
    \caption{Score obtained by each approach on each of the metrics for the commonly solved problems.}
    \label{tab:fairness_table3}
\end{table*}

\end{document}